\documentclass[letter, 12pt]{article}
\pdfoutput=1
\usepackage[dvipsnames]{xcolor}
\usepackage{stackrel}
\usepackage{float}
\usepackage{multirow}
\usepackage{array}
\usepackage{authblk}
\usepackage[marginratio=1:1,height=600pt,width=460pt,tmargin=100pt]{geometry}

\RequirePackage{amsmath}
\RequirePackage{amsthm}
\RequirePackage{mathtools}
\RequirePackage{amsfonts}
\RequirePackage{amssymb}
\RequirePackage{xfrac}

\PassOptionsToPackage{full}{textcomp}
\usepackage{newtxtext}
\usepackage[T1]{fontenc}
\usepackage[utf8]{inputenc}

\usepackage[ruled]{algorithm2e}
\RequirePackage{tikz}

\RequirePackage[authoryear]{natbib}
\RequirePackage[colorlinks,citecolor=blue,urlcolor=blue]{hyperref}
\RequirePackage[nameinlink]{cleveref}
\AtBeginDocument{%
\let\ref\Cref
}

\crefname{appendix}{}{}

\RequirePackage{bm}
\RequirePackage{mathrsfs}

\DeclareMathOperator{\Diag}{diag}
\DeclareMathOperator{\Vectorize}{vec}

\DeclarePairedDelimiter{\norm}{\lVert}{\rVert}
\DeclarePairedDelimiter\abs{\lvert}{\rvert}

\DeclarePairedDelimiterX{\set}[2]{\{}{\}}{\,{#1}:{#2}\,}
\newcommand{\sequence}[2][n]{\left\{\cramped{#2}\right\}_{{#1}=1}^{\infty}}
\newcommand{\entrydivide}{\oslash}
\newcommand{\entrydot}{\odot}
\newcommand{\indicator}[1]{\mathbf{1}_{\{#1\}}}
\newcommand{\numset}[1]{\ensuremath{\mathbb{#1}}}
\newcommand{\prob}{\mathbb{P}}

\newcommand{\given}{\mid}
\newcommand{\loglikelihood}{\ell}
\newcommand{\Ploglikelihood}{\Tilde{\ell}}

\newcommand{\Loss}{\text{Loss}}

\newcommand{\Normal}[1][(\mu,\cramped{\sigma^2})]{\textsf{\upshape N}#1}
\newcommand{\AL}[1][(m,\kappa,\lambda)]{\textsf{\upshape AL}#1}

\newcommand{\Qts}[1][\vec{\theta}]{\operatorname{Q}_{#1}}
\newcommand{\Qtsn}[1][\vec{\theta}]{\hat{\operatorname{Q}}_{#1}}
\newcommand{\Qtse}[1]{\operatorname{Q}_{\theta_{#1}}}

\DeclarePairedDelimiterX{\ointerval}[2]{(}{)}{{#1},{#2}}
\DeclarePairedDelimiterX{\cinterval}[2]{[}{]}{{#1},{#2}}
\DeclarePairedDelimiterX{\ocinterval}[2]{(}{]}{{#1},{#2}}
\DeclarePairedDelimiterX{\cointerval}[2]{[}{)}{{#1},{#2}}

\newcommand{\diff}{\mathop{}\!\mathrm{d}}
\let\oldpartial\partial
\renewcommand{\partial}{\mathop{}\!\oldpartial}
\newcommand{\gradient}{\nabla}

\newcommand{\alms}{\text{a.s.}}

\newcommand{\almscon}{\overset{\alms}{\rightarrow}}

\renewcommand{\ge}{\geqslant}
\renewcommand{\le}{\leqslant}

\newtheorem{thm}{Theorem}

\newtheorem{lem}[thm]{Lemma}

\newtheorem{assump}{Assumption}

\theoremstyle{definition}

\crefname{assump}{Assumption}{Assumptions}

\renewcommand{\vec}[1]{\bm{#1}}
\newcommand{\mat}[1]{\bm{#1}}

\makeatletter
\newcommand*{\transpose}{%
  {\mathpalette\@transpose{}}%
}
\newcommand*{\@transpose}[2]{%
  \raisebox{\depth}{$\m@th#1\intercal$}%
}
\makeatother

\newcommand{\unitvector}{\ensuremath{\vec{e}}}
\newcommand{\onevector}{\ensuremath{\vec{1}}}

\let\originalleft\left
\let\originalright\right
\renewcommand{\left}{\mathopen{}\mathclose\bgroup\originalleft}
\renewcommand{\right}{\aftergroup\egroup\originalright}
\newcommand{\zerodel}{.\kern-\nulldelimiterspace}

\makeatother

\title{Ensemble quantile classifier}
\author[1]{Yuanhao Lai}
\author[1]{Ian McLeod}
\affil[1]{\textit{Department for Statistical and Actuarial Sciences, Western University, Canada}}
\date{Sept 25, 2019}

\begin{document}
\maketitle

\begin{abstract}
\noindent 
Both the median-based classifier and the quantile-based classifier are useful
for discriminating high-dimensional data with
heavy-tailed or skewed inputs.
But these methods are restricted as they assign equal weight to each variable
in an unregularized way.
The ensemble quantile classifier is a more flexible
regularized classifier that provides better performance
with high-dimensional data, asymmetric data
or when there are many irrelevant extraneous inputs.
The improved performance is demonstrated by a simulation
study as well as an application to text categorization.
It is proven that the estimated parameters of the ensemble quantile classifier
consistently estimate the minimal population loss under suitable general model assumptions.
It is also shown that the ensemble quantile classifier
is Bayes optimal under suitable assumptions with asymmetric Laplace distribution
inputs.
\end{abstract}


\section{Introduction}

The class prediction problem with
$p$-dimensional input $\vec{x}=(x_1,\dotsc,x_p)$ and
output variable $y \in {\cal K}$,  where ${\cal K}= \{1,\ldots,K\}$,
is considered.
For each class $k\in {\cal K}$, $\vec{x}$ is a $p$-multivariate random vector generated from
the multivariate probability distributions $P_k$.
The family of component-wise distance-based discriminant rules is defined by,
\begin{equation}
D_{k}=\sum_{j=1}^{p} d(x_j, P_{k,j}),
\label{eq:discriminantFunction}
\end{equation}
where $x_j \in \Re$ is a test input, $P_{k,j}$ is the $j$-th
marginal distribution of $P_k$, and
$d(x_j,P_{k,j})$ is the distance between $x_j$ and $P_{k,j}$
\citep{Hennig2016,  HALL2009, TIBSHIRANI2003}.
The optimal prediction is
\begin{equation}
\hat{y} = {\rm argmin}_{k \in {\cal K}} D_k.
\label{eq:decisionRule}
\end{equation}
For example, the centroid classifier may be defined by
$d(x_j, P_{k,j}) =  (x_j - \mu_{k,j})^2$ where $\mu_{k,j}$ is the mean of $P_{k,j}$.
This classifier is a special case of the naive Bayes classifier
 \citep{tibs2009},
also known as the diagonal linear discriminant classifier.
It provides an effective classifier for large $p$ and many types of
high dimensional data inputs \citep{Dudoit2002, Bickel2004, Fan2008}.
When the input $\vec{x}$ includes symmetric random variables
with fat tails, the median classifier (MC),
$d(x_j, P_{k,j}) = | x_j - m_{k,j} |$, where $m_{k,j} = {\rm median}(P_{k,j})$,
$k \in {\cal K}$ and $j=1,\ldots,p$, often has better performance.
\citet{HALL2009} proved that under suitable regularity conditions
MC produced asymptotically correct predictions.
In practice a training data set,
where variables may be rescaled if necessary \citep[Section 4.1]{Hennig2016},
is used to estimate the parameters $\mu_{k,j}$
or $m_{k,j}$ for $k \in {\cal K}$ and $j=1,\ldots,p$.

It sometimes happens that the distribution of two variables is similar near the center
but differs in the tails due to skewness or other characteristics.
Quantile regression makes use of this phenomenon \citep{Koenker1978}. 
The Tukey mean difference plot \citep[p.21]{cleveland1993visualizing} was invented to
compare data from such distributions.
\citet{Hennig2016} extended MC to the quantile-based classifier (QC) defined by
$d(x_j, P_{k,j}) = \rho_{\theta_j}(x_j - q_{k,j}({\theta_j}))$,
where $q_{k,j}({\theta_j})$ is the $\theta_j$-quantile of $P_{k,j}$ for
$0<\theta_j<1$ and
\begin{equation}
\rho_\theta(u)=u(\theta-\indicator{u<0})
\label{eq:QuantileDistance}
\end{equation}
is the quantile distance function \citep{Koenker1978,Koenker2005}.
When $\theta=0.5$, QC reduces to MC.
\citet{Hennig2016} showed that the QC can provide the Bayes optimal prediction
with skewed input distributions.
The usefulness of QC was demonstrated by simulation as well as an application
\citep{Hennig2016}.
An R package which implements the centroid, median and quantile classifiers
is available \citep{quantileDA}.

Although QC is effective for discriminating
high-dimensional data with heavy-tailed or skewed inputs,
it suffers from the restriction of assigning
each variable the same importance, which limits its effectiveness when
there are irrelevant extraneous inputs.
Another limitation for QC and the median centroid classifier
with high dimensional data may be noise accumulation.
\citet[Theorem 1a]{Fan2008} proved that the centroid classifier may perform no better than
random guessing due to noise accumulation with high dimensional
data.

Our proposed ensemble quantile classifier (EQC), presented in \ref{sec:EQC},
is a flexible regularized classifier that aims to overcome these two limitations and
provides better performance with high-dimensional data, asymmetric data or when there are many irrelevant extraneous inputs.
We introduce the binary EQC for discriminating observations into one of two classes and then extend it to situations with more than two classes.
In \ref{thm:consitencyEmpClassifier} of \ref{sec:theory},
it is shown that sample loss function of EQC
converges to the population value when the sample size increases.
In \ref{sec:simulation}  and \ref{sec:application},
the improved performance of EQC is demonstrated by
a simulation study and an application to text categorization.

\newpage
\section{Ensemble quantile classifier}
\label{sec:EQC}

Ensemble predictors were derived from the idea popularly known
as {\it Wisdom of the Crowd} \citep{tibs2009, silver2012}.
\citet{Newbold} showed that economic time series forecasts could
be improved by using a weighted average of forecasts from a
heterogeneous variety of time series models.
Many advanced ensemble prediction methods for supervised learning problems have
been developed such as random forests \citep{breiman2001random} and
various boosting algorithms \citep{freund1997decision, schapire2012}.
The ensemble stacking method introduced by \citet{wolpert1992stacked} and \citet{breiman1996stacked}
has also been widely used.
Comprehensive surveys of ensemble learning algorithms are given by
\citet{tibs2009, dietterich2000ensemble, zhou2012ensemble} and \citet{lior}.
Ensemble stacking uses a metalearner to combine base learners.
In \ref{subsec:QC} and \ref{subsec:EQC binary} a method for using regularized logistic regression
to combine quantile classifiers is developed and is generalized
to the multiclass case in \ref{subsec:EQC}.

\subsection{Quantile-difference transformation and EQC}
\label{subsec:QC}

For the classification problem with $K$ classes and $p$ inputs, let $q_{k,j}({\theta_j})$
be the $\theta_j$-quantile of $P_{k,j}$ where $0<\theta_j<1$ for $k \in {\cal K}$ and $j=1,\ldots,p$.
The derived inputs to the metalearner are obtained from the quantile-difference transformation of $\vec{x}=(x_1,\dotsc,x_p)$ defined by,
\begin{equation}
\Qts^{(k_1,k_2)}(\vec{x})=\bigl(\Qtse{1}^{(k_1,k_2)}(x_1),\dotsc,\Qtse{p}^{(k_1,k_2)}(x_p)\bigr),
\label{eq:quantileTransform}
\end{equation}
where,
\[
 \Qtse{j}^{(k_1,k_2)}(x_j)=
 \rho_{\theta_j}(x_j- q_{k_1,j}(\theta_j)) - \rho_{\theta_j}(x_j- q_{k_2,j}(\theta_j)),
 \ j=1,\dotsc,p,
\]
and $\rho_\theta(u)=u(\theta-\indicator{u<0})$ is the quantile-distance function.
The superscript $(k_1,k_2)$ is omitted if $k_{1}=1$ and $k_{2}=2$.
As shown in \ref{fig:Qtransform} the quantile-difference transformation has constant
tails so the derived inputs are insensitive to outliers.

\begin{figure}[H]
\begin{center}
\tikzset{every picture/.style={line width=0.75pt}} 
\begin{tikzpicture}[x=0.75pt,y=0.75pt,yscale=-1,xscale=1]

\draw  [dash pattern={on 0.84pt off 2.51pt}]  (131.87,71.43) -- (131.87,126.02) ;
\draw    (43.36,71.97) -- (326.91,71.97) ;
\draw   (317.5,67.27) -- (328.31,72.05) -- (317.5,76.83) ;
\draw  [dash pattern={on 0.84pt off 2.51pt}]  (233.69,25.71) -- (233.69,67) ;
\draw    (131.87,67) -- (131.87,71.43) ;
\draw    (233.69,67) -- (233.69,71.43) ;
\draw    (133.43,126.02) -- (233.69,25.71) ;
\draw    (233.69,25.71) -- (327.69,25.71) ;
\draw    (39.44,126.02) -- (133.43,126.02) ;

\draw   (332.39,127.6) .. controls (337.06,127.6) and (339.39,125.27) .. (339.39,120.6) -- (339.39,87.85) .. controls (339.39,81.18) and (341.72,77.85) .. (346.39,77.85) .. controls (341.72,77.85) and (339.39,74.52) .. (339.39,67.85)(339.39,70.85) -- (339.39,35.11) .. controls (339.39,30.44) and (337.06,28.11) .. (332.39,28.11) ;

\draw (133.43,81.02) node [scale=0.9]  {$q_{1}( \theta )$};
\draw (237.61,81.02) node [scale=0.9]  {$q_{2}( \theta )$};
\draw (318.77,81.8) node   {$x$};
\draw (96.74,138.55) node [scale=0.9]  {$-[ q_{2}( \theta ) -q_{1}( \theta )]( 1-\theta )$};
\draw (283.34,12.57) node [scale=0.9]  {$[ q_{2}( \theta ) -q_{1}( \theta )] \theta $};
\draw (393.03,79.02) node [scale=0.9]  {$q_{2}( \theta ) -q_{1}( \theta )$};
\end{tikzpicture}
\end{center}
\caption{Quantile-difference transformation for classes 1 and 2 when $q_{1}( \theta ) < q_{2}( \theta )$.}
\label{fig:Qtransform}
\end{figure}
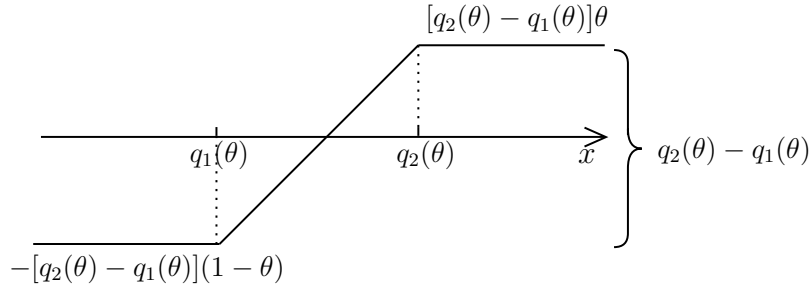

In the binary case, the QC discriminant function is given by
$s(\vec{x} \mid \vec{\theta})=\sum_j \Qtse{j}(x_j)$
where $\vec{x}=(x_1,\ldots,x_p)$ is a test input and
$\vec{\theta}=(\theta_1, \ldots, \theta_p)$.
The classifier predicts class 1 or 2 according as
$s(\vec{x} \mid \vec{\theta})\le 0$ or $>0$.
\citet{Hennig2016} estimated the parameter $\vec{\theta}$ by minimizing
the misclassification rate on the training data using a grid search.
In most cases they found that using the same value of $\theta_j=\theta$,
$j=1,\dotsc,p$
for all input variables worked well for the QC,
which means a restriction $\vec{\theta}= \{\theta,\dotsc,\theta\} \in \vec{\Theta}\subseteq (0,1)^p$.
For simplicity and computational expediency,
this restriction was imposed in the simulation study and the application to text categorization.


\subsection{EQC for binary case}
\label{subsec:EQC binary}

The discriminant function for QC is simply an additive
sum $\Qtse{j}(x_j)$ for $j=1,\dotsc,p$ but
in practice it is often the case that several of the variables
are more important and should be given more weights.
EQC is proposed to extend QC by providing an effective classifier
that takes this into account.
The discriminate function for the EQC binary case may be written,
\begin{equation}
s(\vec{x} \mid \vec{\theta},\beta_0,\vec{\beta})=
\mathsf{C}\bigl(\Qts(\vec{x}) \mid \beta_0, \vec{\beta}\bigr),
\label{eq:EQCdecisionBound}
\end{equation}
where $\Qts(\vec{x})$ is defined in \ref{eq:quantileTransform}
and $\mathsf{C}(\vec{z} \mid \beta_0,\vec{\beta})$ is the metalearner
with the intercept term $\beta_0\in\numset{R}$ and
the weight vector $\vec{\beta}\in\numset{R}^p$.
Then $(\beta_0,\vec{\beta})$ along with the $p$ quantile parameters $\vec{\theta}\in (0,1)^p$
may be estimated by minimizing a suitable regularized loss function
with a regularization parameter $\vec{\alpha}$
using cross-validation.

The metalearner $\mathsf{C}(\vec{z} \mid \beta_0,\vec{\beta})$
can be substituted by the discriminant function of most regularized classifiers
such as the penalized logistic regression \citep{park2007penalized} or
the support vector machine (SVM) \citep{cortes1995support}.
For the penalized logistic regression $\vec{\alpha}=\lambda$,
where $\lambda$ is the penalty defined in \ref{eq:EBinomialDevianceL2}
while for the SVM model with the linear kernel
defined in \ref{eq:EHinge},
$\vec{\alpha=}\mathfrak{c}$, where $\mathfrak{c}$ is the cost penalty.
Ridge logistic regression is recommended as 
a default choice for $\mathsf{C}$ since
it often performs well.
For high-dimensional data where $p>n$,
it is preferable to treat $\vec{\theta}$ as a tuning
parameter and estimate it together with $\vec{\alpha}$ using
cross-validation to avoid overfitting.

When the quantiles are substituted by their estimates, the estimated
discriminant function is denoted by
$\hat{\mathcal{G}}(\vec{x} \mid \vec{\theta}, \beta_0,\vec{\beta})$
and the estimated quantile-difference transformation is denoted by
$\Qtsn(\vec{x}_i)$.

Using the penalized logistic regression for $\mathsf{C}$,
\begin{equation}
\mathsf{C}\bigl(\Qts(\vec{x}) \mid  \beta_0, \vec{\beta}\bigr)=
\beta_0+\sum_{j=1}^p \beta_j\Qtse{j}(x_j).
\label{eq:EQCdecisionBound_linear}
\end{equation}
Let $\vec{\alpha}=\lambda$ be the penalty parameter in the regularized binomial loss function \citep{Friedman2010}.
So given $\lambda$ and the input $\Qts(\vec{x})$,
$(\beta_0,\vec{\beta})$ may be estimated by minimizing,
\begin{align}
\Loss_{\text{binomial}} \Bigl(\beta_0,\vec{\beta} \bigm| \lambda, \Qts(\vec{x}_i), i=1,\dotsc,n \Bigr)=
& -\frac{1}{n}\sum_{i=1}^{n}
\Bigl\{ (y_i-1) \mathsf{C}\bigl(\Qts(\vec{x}_i) \mid \beta_0,\vec{\beta}\bigr) -
 \nonumber \\
\phantom{{}=}& \log (1+e^{ \mathsf{C}(\Qts(\vec{x}_i) \mid \beta_0,\vec{\beta}) })\Bigr\}
 +  \frac{\lambda}{2} \norm{\vec{\beta}}_l,
\label{eq:EBinomialDevianceL2}
\end{align}
where $\norm{\vec{\beta}}_1=\sum_{j=1}^{p} |\beta_{j}|$ for LASSO
and $\norm{\vec{\beta}}_2=\sum_{j=1}^{p} \beta_{j}^2$ for ridge regression.
Using the SVM with the linear kernel (LSVM) has the same linear
discriminant function as \ref{eq:EQCdecisionBound_linear},
but $(\beta_0,\vec{\beta})$ is estimated by minimizing the regularized
hinge loss \citep[Equation 12.25]{tibs2009},
\begin{align}
\Loss_{\text{hinge}} \Bigl(\beta_0,\vec{\beta} \bigm| \mathfrak{c}, \Qts(\vec{x}_i), i=1,\dotsc,n \Bigr) =
& \frac{1}{n}\sum_{i=1}^{n}
\Bigl[ 1 - (y_i-1) \mathsf{C}
\bigl(\Qts(\vec{x}_i) \mid \beta_0,\vec{\beta} \bigr)  \Bigr]_{+}
+  \frac{1}{2n \mathfrak{c}} \norm{\vec{\beta}}_2,
\label{eq:EHinge}
\end{align}
where $[x]_+$ indicates the positive part of $x$ and
$\mathfrak{c}$ is the cost tuning parameter.

If $\beta_0=0$ and $\beta_j=1$ for $j=1,\dotsc,p$,
then EQC has the same decision boundary as QC.
In \ref{appendix:ASL}, it is shown that EQC with $\mathsf{C}$
defined in \ref{eq:EQCdecisionBound_linear}
has the same form as the Bayes decision boundary when $P_1$ and $P_2$
consist of independent asymmetric Laplace distributions.
This motivates further exploration and development of the EQC.
The estimation of $(\beta_0,\vec{\beta})$
by ridge/LASSO penalized logistic regression and LSVM
are all capable of dealing with high dimensional data.
The associated ensemble classifiers used in this paper
are denoted respectively by EQC/RIDGE, EQC/LASSO and EQC/LSVM.
A non-negative constraint of $\vec{\beta}$ was also investigated
but we did not find an experimentally significant accuracy improvement, which
agrees with a previous study of stacking classifiers \citep{ting1999issues}.

\ref{alg:Algorithm_EQC} shows the entire process of tuning and training EQC.
Here the misclassification rate is used as a criterion to choose the tuning parameters
but in some cases other criteria such as the AUC may be appropriate.

The time complexity of this algorithm is determined by the time complexities 
for the quantile estimation,
the quantile-difference transformation and the coordinate descent algorithm
which are respectively
$O\bigl(p n \log(n) \bigr)$,
$O(n p H)$,
$O\bigl(n (p+1) I H)\bigr)$,
where $H$ is the size of the tuning set $\vec{\Theta}$,
and $I$ is the number
of iterations required for minimization of the loss function.
In total \ref{alg:Algorithm_EQC} has complexity $O\bigl((T+1)(pn\log(n)+n(p+1)IH)\bigr)$,
where $T$ is the number of cross-validation folds.
The computational burden for cross-validation may be
reduced by using parallel computation \citep{kuhn2013}.

\begin{algorithm}[ht]
\KwData{ $S=\{(\vec{x}_1, y_1),\dotsc,(\vec{x}_n, y_n)\}$, $p=$ the number of variables within $\vec{x}$. }
\KwIn{ $\vec{\Theta}=$ tuning set of $\vec{\theta}\in (0,1)^p$,
$\numset{A}=$ tuning set of $\vec{\alpha}$,
$T=$ the number of cross-validation folds.}
\Begin(tuning parameters:){
Randomly partition $S$ into $T$ non-overlap folds, $S_1,\dotsc,S_T$. \\
\For(fit $\mathsf{C}$ on $S\setminus S_t$,){$t = 1,\dotsc,T$}{
  \ForEach{$\vec{\theta}$ in $\vec{\Theta}$}{
   Estimate $\hat{q}_{k,j}(\theta_j)$ for $k=1,2$ and $j=1,\dotsc,p$ from $S\setminus S_t$. \\
   Compute $\Qtsn(\vec{x}_i)$ for $i=1,\dotsc,n$. \\
    \ForEach{ $\vec{\alpha}$ in $\numset{A}$ }{
         Estimate $(\beta_0,\vec{\beta})$ by minimizing the loss function
         such as
         \ref{eq:EBinomialDevianceL2} on
         $\Qtsn(\vec{x}_i)$ for $(\vec{x}_i,y_i) \in S\setminus S_t$
         with the tuning parameter $\vec{\alpha}$. \\
        Apply the estimated $\hat{\mathsf{C}}$ on $\Qtsn(\vec{x}_i)$ for $(\vec{x}_i,y_i) \in S_t$ and return
        the misclassification rates.
      }
  }
 }
 Average the misclassification rate over folds for each combination of $\vec{\theta}$ and $\vec{\alpha}$.
 Denote $(\hat{\vec{\theta}}, \hat{\vec{\alpha}})$ as the combination with the minimum average misclassification rate. \\
}
\KwOut{Refit $\mathsf{C}$ with $(\hat{\vec{\theta}}, \hat{\vec{\alpha}})$
 on the full data $S$.}
\caption{Tune and train EQC for binary case}
\label{alg:Algorithm_EQC}
\end{algorithm}

\subsection{Multiclass EQC}
\label{subsec:EQC}

A practical method to extend the binary classifier to multiclass ($K>2$)
is to build a set of
one-versus-all classifiers or a set of one-versus-one classifiers \cite[p. 658]{tibs2009}.
A less heuristic approach, similar to the multinomial logistic regression,
is to use the $K-1$ log-odd-ratios
to implement maximum likelihood estimation (MLE).
The multinomial logistic regression requires estimation of $(K-1)(p+1)$ coefficients
but here the multiclass EQC only requires $p+K-1$ coefficients
including $p$ weights $\beta_j$, $j=1,\dotsc,p$,
and $K-1$ intercept terms $\beta_{0,k}$, $k=1,\dotsc,K-1$.

Let $\vec{\beta}=(\beta_1,\dotsc,\beta_p)$.
Assume for an input $\vec{x}$,
\begin{align*}
\log \frac{\prob(y=1 \given
\vec{x},\vec{\theta},\vec{\beta},\{\beta_{0,k}\}_{k=1}^{K-1})}
{\prob(y=K \given
\vec{x},\vec{\theta},\vec{\beta},\{\beta_{0,k}\}_{k=1}^{K-1})}
& = -\mathsf{C}\bigl(\Qts^{(1,K)}(\vec{x}) \mid
 \beta_{0,1},\vec{\beta}\bigr) \\
\log \frac{\prob(y=2 \given \vec{x},\vec{\theta},\vec{\beta},\{\beta_{0,k}\}_{k=1}^{K-1})}
{\prob(y=K \given
\vec{x},\vec{\theta},\vec{\beta},\{\beta_{0,k}\}_{k=1}^{K-1})}
& = -\mathsf{C}\bigl(\Qts^{(2,K)}(\vec{x}) \mid
\beta_{0,2},\vec{\beta}\bigr) \\
\vdots & \\
\log \frac{\prob(y=K-1 \given \vec{x},\vec{\theta},\vec{\beta},\{\beta_{0,k}\}_{k=1}^{K-1})}
{\prob(y=K \given
\vec{x},\vec{\theta},\vec{\beta},\{\beta_{0,k}\}_{k=1}^{K-1})}
& = -\mathsf{C}\bigl(\Qts^{(K-1,K)}(\vec{x}) \mid
\beta_{0,K-1},\vec{\beta}\bigr) ,
\end{align*}
and $\sum_{k=1}^{K} \prob(y=k \given \vec{x},\vec{\theta},\vec{\beta},\{\beta_{0,k}\}_{k=1}^{K-1})=1$.
The negative sign prior to $\mathsf{C}$ is used because
class $K$ is used in the denominator of the log-odd-ratios
and it is the alternative class in $\Qts^{(k,K)}(\vec{x})$,
which implies that
the smaller $\mathsf{C}\bigl(\Qts^{(k,K)}(\vec{x}) \mid
\beta_{0,k},\vec{\beta}\bigr)$ is,
the closer $\vec{x}$ is to class $k$ compared to class $K$ and
hence the larger the log-odd-ratios between class $k$ and class $K$ is.
Let $\beta_{0,K}=0$ and thus,
\begin{equation}
\prob(y=k_0 \given
\vec{x},\vec{\theta},\vec{\beta},\{\beta_{0,k}\}_{k=1}^{K-1})
=\frac{e^{-\mathsf{C}\bigl(\left. \Qts^{(k_0,K)}(\vec{x}) \right\vert\,
\beta_{0,k_0}, \vec{\beta}\bigr)}}
{\sum_{k=1}^{K}
e^{{ -\mathsf{C}\bigl(\left.\Qts^{(k,K)}(\vec{x}) \right\vert\,
\beta_{0,k}, \vec{\beta}\bigr)}} },
\text{ for } k_0=1,\dotsc,K.
\label{eq:softmaxProb}
\end{equation}

Given the tuning parameter $\lambda$,
the L2 regularized log-likelihood function may be written,
\begin{equation}
\Ploglikelihood(\vec{\beta},\{\beta_{0,k}\}_{k=1}^{K-1} \given
\vec{\theta},\lambda)=
\frac{1}{n}\sum_{i=1}^{n} \log
\prob(y=y_{i} \given \vec{x}_i,\vec{\theta},
\vec{\beta},\{\beta_{0,k}\}_{k=1}^{K-1} )
- \frac{\lambda}{2}\sum_{j=1}^{p} \beta_{j}^2.
\label{eq:Ploglik}
\end{equation}
It can be shown that
$\Ploglikelihood(\vec{\beta},\{\beta_{0,k}\}_{k=1}^{K-1} \given
\vec{\theta},\lambda)$ is a concave function so it is amenable to optimization based on gradients or Newton's method.
For further details see \ref{appendix:MLEmulticlassEQC} and
the software implementation \citep{eqcGithub}.

\section{Asymptotic consistency}
\label{sec:theory}
In this section, the theoretical result is derived in a slightly modified setup of
the method in \ref{alg:Algorithm_EQC}.
It is assumed that $p$ is fixed while $n$ increases,
so $\vec{\theta}$ and $(\beta_0,\vec{\beta})$ may be estimated by maximum likelihood.
In addition, $\vec{\alpha}$ is neglected as
the asymptotic properties of the selection of the tuning parameter are not discussed.

Let $(\tilde{\vec{\theta}},\tilde{\beta}_0,\tilde{\vec{\beta}})$ be the parameters that
minimize the population binomial loss function,
\begin{align}
\Psi(\vec{\theta},\beta_0,\vec{\beta}) & =
\pi_{1}\int \log (1+e^{ \mathsf{C}(\Qts(\vec{x}) \mid \beta_0,\vec{\beta}) })  \diff P_{1}(\vec{x})
 -  \nonumber \\
& \phantom{==}
\pi_{2}\int \Bigl\{ \mathsf{C}\bigl(\Qts(\vec{x}) \mid \beta_0,\vec{\beta}\bigr) -
\log (1+e^{ \mathsf{C}(\Qts(\vec{x}) \mid \beta_0,\vec{\beta}) })\Bigr\}  \diff P_{2}(\vec{x}),
\label{eq:BinomialDeviance}
\end{align}
where $\pi_{1}$ and $\pi_{2}$ are prior probabilities of the two classes.

Let $(\hat{\vec{\theta}}_n,\hat{\beta}_{n0},\hat{\vec{\beta}}_n)$ be the parameters that minimize
the empirical binomial loss function,
\begin{equation}
\Psi_n(\vec{\theta},\beta_0,\vec{\beta})=
-\frac{1}{n}\sum_{i=1}^{n}
\Bigl\{ (y_i-1) \mathsf{C}\bigl(\Qtsn(\vec{x}_i) \mid \beta_0,\vec{\beta}\bigr) -
\log (1+e^{ \mathsf{C}(\Qtsn(\vec{x}_i) \mid
\beta_0,\vec{\beta}) })\Bigr\}.
\label{eq:EBinomialDeviance}
\end{equation}

It is shown that under suitable assumptions,
$(\hat{\vec{\theta}}_n,\hat{\beta}_{n0},\hat{\vec{\beta}}_n)$
is a consistent estimator of $(\tilde{\vec{\theta}},\tilde{\beta}_0,\tilde{\vec{\beta}})$.
The proofs are available in \ref{appendix:proof}.
These results have been proved by \citet{Hennig2016} for the
quantile-based classifier with the 0-1 loss function.
The proof given by them has been adapted to
take into account the additional parameters $(\beta_0,\vec{\beta})$ and
the change of the loss function from the 0-1 loss function to
the binomial loss function.
\ref{assumption:boundedWeight} is added in addition to
\ref{assumption:continuousQuantile} made by \citet{Hennig2016}.
The linear discriminant function or metalearner $\mathsf{C}(\vec{z} \mid \beta_0,\vec{\beta})$ in \ref{eq:EQCdecisionBound_linear},
the discriminant function with multiplicative interactions,
and the polynomial discriminant function used in polynomial kernel SVM
all satisfy \ref{assumption:boundedWeight}.
These assumptions ensure the convergence can still hold with $(\beta_0,\vec{\beta})$.
The use of the binomial loss function simplifies the proof and the computation.
Since the 0-1 loss function is not a convex or a continuous function,
its minimization is NP-hard and hence
the binomial loss function or the hinge loss function are used instead.
Assumption 2 of \citet{Hennig2016} is not needed because the binomial loss function is used.

\begin{assump}
\label{assumption:continuousQuantile}
$\forall j=1,\dotsc,p$, $k=1,2$, the quantile function $q_{k,j}(\theta_j)$
is a continuous function of $\theta_j \in \Theta_j \subset \ointerval{0}{1}$.
\end{assump}

\begin{assump}
\label{assumption:boundedWeight}
$\mathsf{C}(\vec{z} \mid \beta_0,\vec{\beta})$ is required to be
differentiable with respect to $\vec{z}$, $\beta_0$ and $\vec{\beta}$.
In addition, $\tilde{\beta}_0$ and $\tilde{\vec{\beta}}$ are required to be bounded.
That is, $\exists C>0$
such that $\abs{\tilde{\beta}_j}\le C$, for $j=0,1,\dotsc,p$.
\end{assump}

\begin{thm}
\label{thm:consitencyClassifier}
Under \ref{assumption:continuousQuantile,assumption:boundedWeight},
$\forall \epsilon>0$,
\[
\lim_{n\rightarrow\infty}\prob\{
\abs{
\Psi(\tilde{\vec{\theta}},\tilde{\beta}_0,\tilde{\vec{\beta}})-
\Psi(\hat{\vec{\theta}}_n,\hat{\beta}_{n0},\hat{\vec{\beta}}_n)
}>\epsilon
\}=0.
\]
\end{thm}
\ref{assumption:boundedWeight} is needed to ensure that the estimation of $(\beta_0,\vec{\beta})$ converges.
\ref{thm:consitencyClassifier} shows that the estimated parameters are consistent in
achieving the minimal population loss.
Beside, \ref{thm:consitencyEmpClassifier} states that
the empirical minimal loss will
converge to the population  minimal loss
asymptotically as $n\rightarrow\infty$ with $p$ fixed.
\begin{thm}
\label{thm:consitencyEmpClassifier}
Under \ref{assumption:continuousQuantile,assumption:boundedWeight},
$\forall \epsilon>0$,
\[
\lim_{n\rightarrow\infty}\prob\{
\abs{
\Psi(\tilde{\vec{\theta}},\tilde{\beta}_0,\tilde{\vec{\beta}})-
\Psi_n(\hat{\vec{\theta}}_n,\hat{\beta}_{n0},\hat{\vec{\beta}}_n)
}>\epsilon
\}=0.
\]
\end{thm}

Based on \ref{thm:consitencyClassifier} and \ref{thm:consitencyEmpClassifier},
when $n$ is large relative to $p$,  \ref{alg:Algorithm_EQC} can be modified
to estimate $\theta$ by minimizing the training loss function instead of using cross-validation approach.

\section{Simulation validation}
\label{sec:simulation}

\subsection{Experimental setup}
\label{sec:simulationDist}

Simulation experiments are presented to demonstrate the improved
performance of EQC over QC with high-dimensional skewed inputs as well
as other classifiers.
The following thirteen classifiers were compared:
\begin{description}
    \item[QC] quantile-based classifier \citep{Hennig2016};
    \item[MC] median-based classifier \citep{HALL2009};
    \item[EMC] EQC with $\theta=0.5$ with ridge logistic regression;
    \item[EQC/LOGISTIC] EQC with logistic regression;
    \item[EQC/RIDGE] EQC with ridge logistic regression;
    \item[EQC/LASSO] EQC with LASSO logistic regression;
    \item[EQC/LSVM] EQC with linear SVM;
    \item[NB] naive Bayes classifier;
    \item[LDA] linear discriminant analysis;
    \item[LASSO] LASSO logistic regression \citep{Friedman2010};
    \item[RIDGE] ridge logistic regression \citep{Friedman2010};
    \item[LSVM] SVM with linear kernel  \citep{cortes1995support};
    \item[RSVM] SVM with radial basis kernel  \citep{cortes1995support}.
\end{description}

Tuning parameters were selected by minimizing the 5-fold cross validation errors.
QC, MC and EQC were fit using the \texttt{R} implementation \citep{eqcGithub}
while NB, LSVM and RSVM used the algorithms in \citet{e1071}.
The LDA from \citep{MASS} was used.
RIDGE and LASSO used the package \textit{glmnet} \citep{Friedman2010}.
EQC/LOGISTIC used the base R function \texttt{stats::glm}.

Three location-shift input distributions, corresponding to heavy-tails,
highly skewed and a heterogeneous skewed, were examined as discussed by \citet{Hennig2016}:
    \begin{description}
    \item[\textbf{T3}] $t$ distribution on 3 degrees of freedom;
    \item[\textbf{LOGNORMAL}] log-normal distribution;
    \item[\textbf{HETEROGENEOUS}] equal number of $W$, $\exp(W)$, $\log(\abs{W})$, $W^2$ and $\abs{W}^{0.5}$ in order,
    where $W \sim \Normal[(0,1)]$.
    \end{description}
All generated variables were statistically independent and
the distributions were adjusted to have mean zero and variance 1.
The classification error rates were estimated using 100 simulations with
independent test samples of size $10^4$.

For each of the three distributions a location-shift vector $\vec{\delta}$
was used to produce the second class where
 $\vec{\delta}=(0.32,\dotsc,0.32)$ for \textbf{T3},
 $\vec{\delta}=(0.06,\dotsc,0.06)$ for \textbf{LOGNORMAL}
and $\vec{\delta}=(0.14,\dotsc,0.14)$ for \textbf{HETEROGENEOUS}.
The additive shifts were chosen to make the test error rate
of the QC close to $10\%$ for samples of size $n=100$.

Simulation experiments to demonstrate the effectiveness of
prediction algorithms with high-dimensional data typically use a
large number of non-informative features or noise variables.
For example, the models of \citet[Equation 18.37]{tibs2009}
and \citet[Section 5.1]{Fan2008} used 95\% and 98\% of the variables
to represent informationless random noise.
We considered the influence of these irreverent variables by including
Gaussian predictors independent of the classes.

For each simulation scenario, the following settings were used,
\begin{enumerate}
   \item Training sample size $n$: $100$, $200$;
   \item Number of all variables $p$: $50$, $100$, $200$;
   \item Standard Gaussian noises with the percentage of noise variables
   within the $p$ variables set to
   $0\%$, $50\%$, $90\%$, which corresponds to 0, $p/2$ and $0.9 \times p$ variables being
   non-informative. The corresponding simulation parameter setting will be denoted as
   ${\rm NOISE}\ = 0\%, 50\%, 90\%$.
   For example, when ${\rm NOISE}\ = 90\%$ there are
   respectively 5, 10 and 20 informative variables when $p=50, 100, 200$.
\end{enumerate}
In addition to the case where the input variables were statistical independent,
the correlated variables case was also investigated.
Correlation was imposed by using the Gaussian copula with
the correlation matrix uniformly sampled from the space of positive-definite
correlation matrices  \citep{JOE2006} with equal correlations distributed as beta($0.5$, $0.5$).
The implementation is available in the R package \textit{clusterGeneration} \citep{clusterGeneration}.

\subsection{Test error rates}
\label{sec:simInterpret}

The mean test error rates for each of the 100 simulations are tabulated
in \ref{appendix:otherDetail}.

The boxplots of the test error rates for the independent variables
in the low dimensional, $n=200, p=50$,
and the high dimensional, $n=100, p=200$, scenarios are displayed in
\ref{fig:Err_n200p50} and \ref{fig:Err_n100p200} respectively.
The scenario with extraneous noise present is shown in
the bottom two rows of \ref{fig:Err_n200p50} and \ref{fig:Err_n100p200} and
here it is seen that in both the  \textbf{LOGNORMAL} and \textbf{HETEROGENEOUS} cases,
the EQC methods outperform all the other methods.

Focusing on \ref{fig:Err_n200p50},
in the symmetric thick-tailed case, \textbf{T3}, MC is best
but QC, EMC and EQC/RIDGE closely approximate the MC performance as might
be expected.
While in the skewed cases, \textbf{LOGNORMAL} and \textbf{HETEROGENEOUS},
the four regularized EQC methods outperform all others.
It is also interesting that
EQC/LOGISITC has a lower error rate than QC in the \textbf{HETEROGENEOUS} case
as shown in the panels in the right-most column.
This implies that the addition of weights using the ensemble method
can help improve performance when the importance of variables varies.
However, even in this case the regularized EQC methods are still best 
and the relative performance of the regularized methods over QC
improves as the proportion of noise variables increases.
Next in the \textbf{LOGNORMAL} case shown in the middle panels in \ref{fig:Err_n200p50},
EQC/RIDGE has overall the best performance though when there
is no extraneous noise, QC is about the same.
But as extraneous noise is added, all EQC methods improve relative to QC
where EQC/LOGISTIC's performance is slightly
worse than the EQC regularized logistic methods.

\begin{figure}[H]
\begin{center}
\includegraphics[scale=0.6]{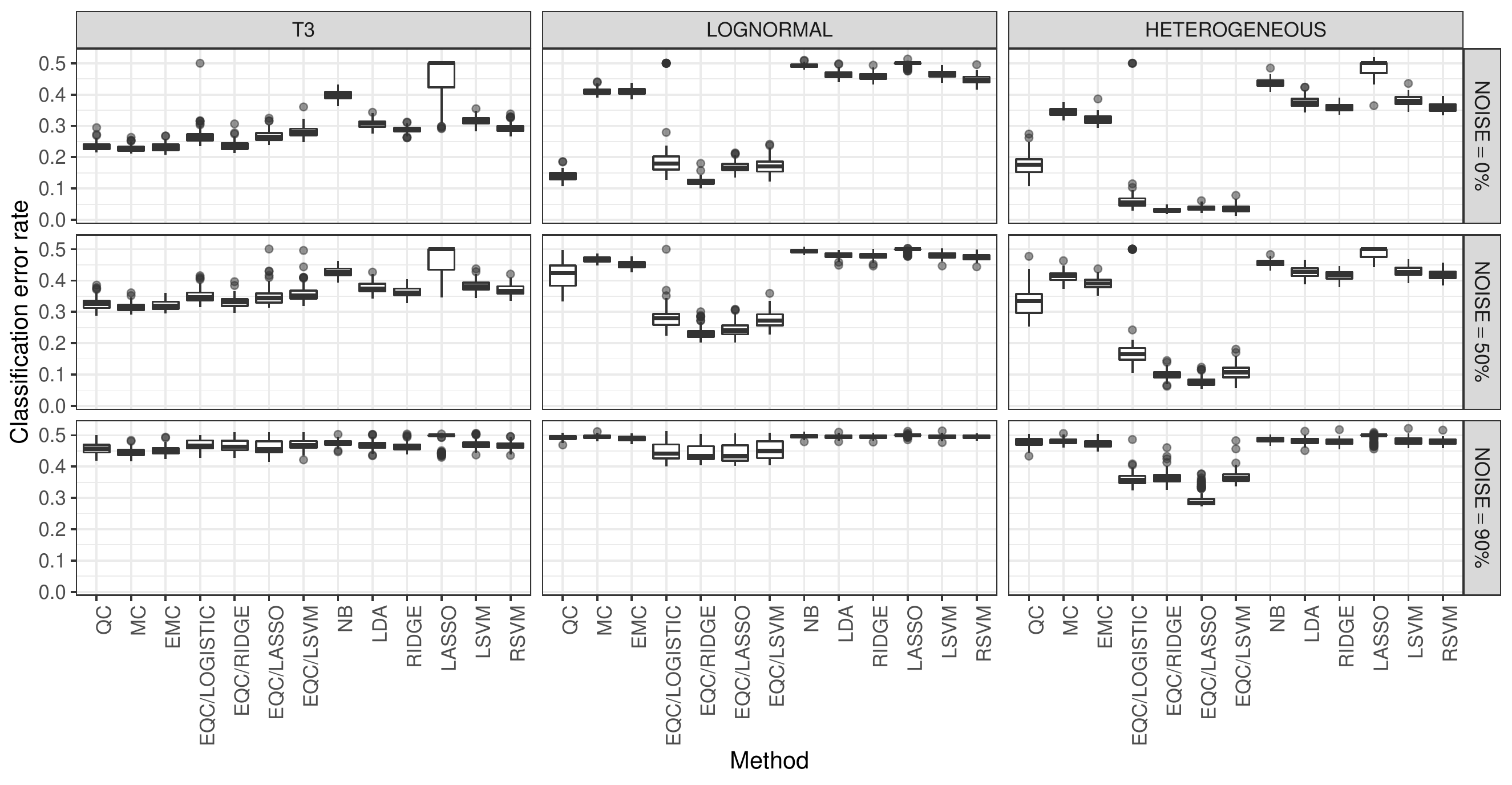}
\end{center}
\caption{Low dimensional scenario test error rates, $n=200, p=50$.}
\label{fig:Err_n200p50}
\end{figure}

In the high-dimensional case in \ref{fig:Err_n100p200},
the conclusions are broadly similar to the low-dimensional case 
in \ref{fig:Err_n200p50} but with two notable differences.
First, QC is much worse than the EQC/RIDGE in
the \textbf{LOGNORMAL} scenario even when all variables are informative.
Since QC lacks regularization, it becomes a victim of
the accumulated noise phenomenon \citep{Fan2008}.
Second, EQC/LASSO is much worse than EQC/RIDGE and EQC/LSVM
with the low (0\%) and medium (50\%) level of noises
since the assumption of sparse predictors made by 
LASSO 
(\citeauthor{tibs2009}, \citeyear{tibs2009}, Section 16.2.2;
\citeauthor{James2013}, \citeyear{James2013}, Section 6.2.2.3)
does not hold.
Conversely when the noise level is 90\%, EQC/LASSO becomes
competitive to EQC/RIDGE and EQC/LSVM in the scenarios
of  \textbf{T3} and  \textbf{LOGNORMAL}, 
and it becomes dominant in the \textbf{HETEROGENEOUS} scenario.

\begin{figure}[H]
\begin{center}
\includegraphics[scale=0.6]{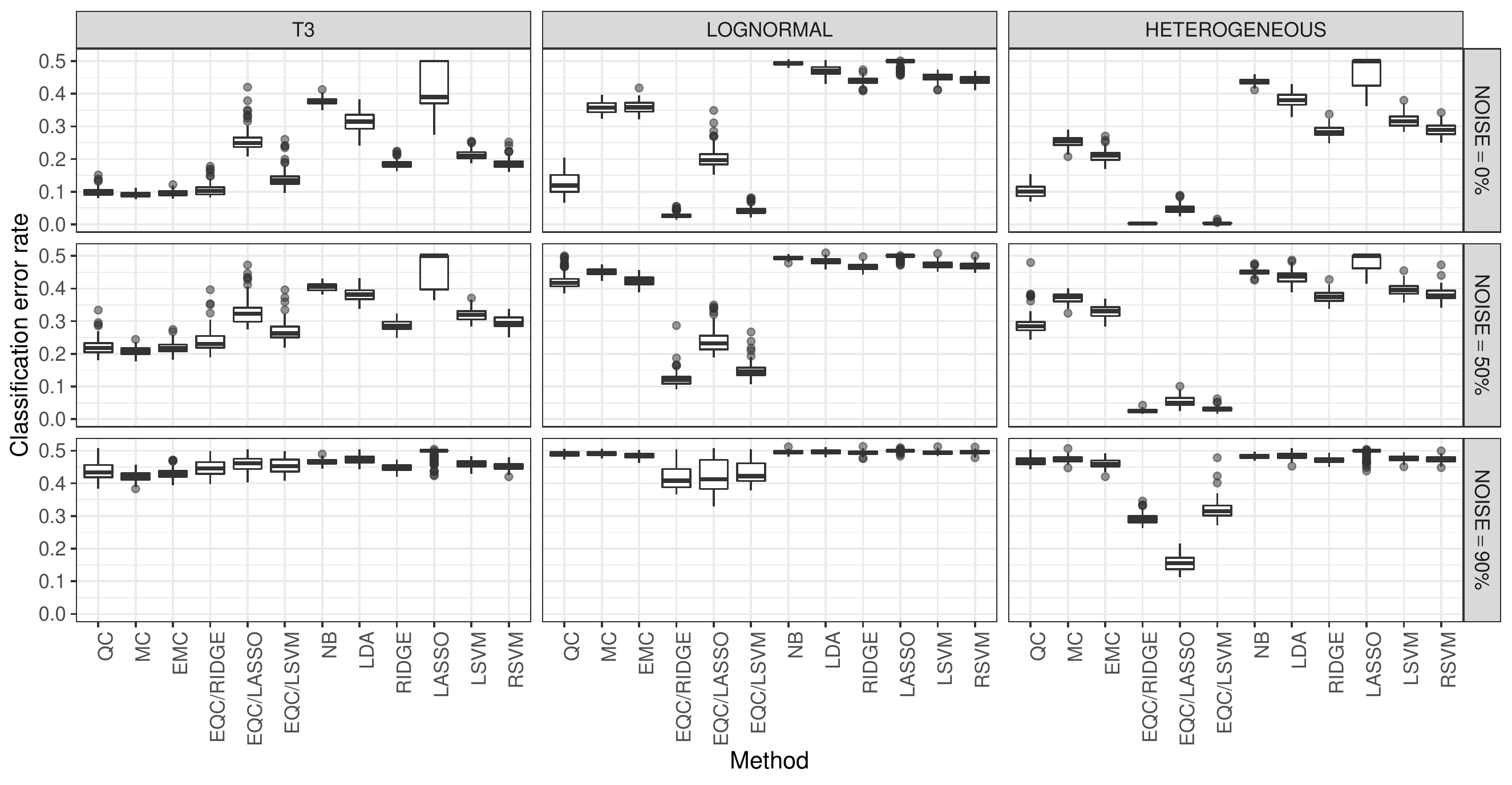}
\end{center}
\caption{High dimensional scenario test error rates, $n=100, p=200$. 
EQC/LOGISTIC is not available in the high dimensional scenario.
}
\label{fig:Err_n100p200}
\end{figure}

\ref{fig:Err_n200p50cor} shows that
the difference in classifier performance between the independent case and the dependent case is negligible 
in the skewed scenarios \textbf{LOGNORMAL} and \textbf{HETEROGENEOUS}.
In the  \textbf{T3} scenario, 
The performance of LDA is best and is greatly improved
over the case with independent variables.
This improvement is not surprising since the correlations
induce heterogeneous weights on variables for the LDA \cite[Equation 4.9]{tibs2009}.

\begin{figure}[H]
\begin{center}
\includegraphics[scale=0.6]{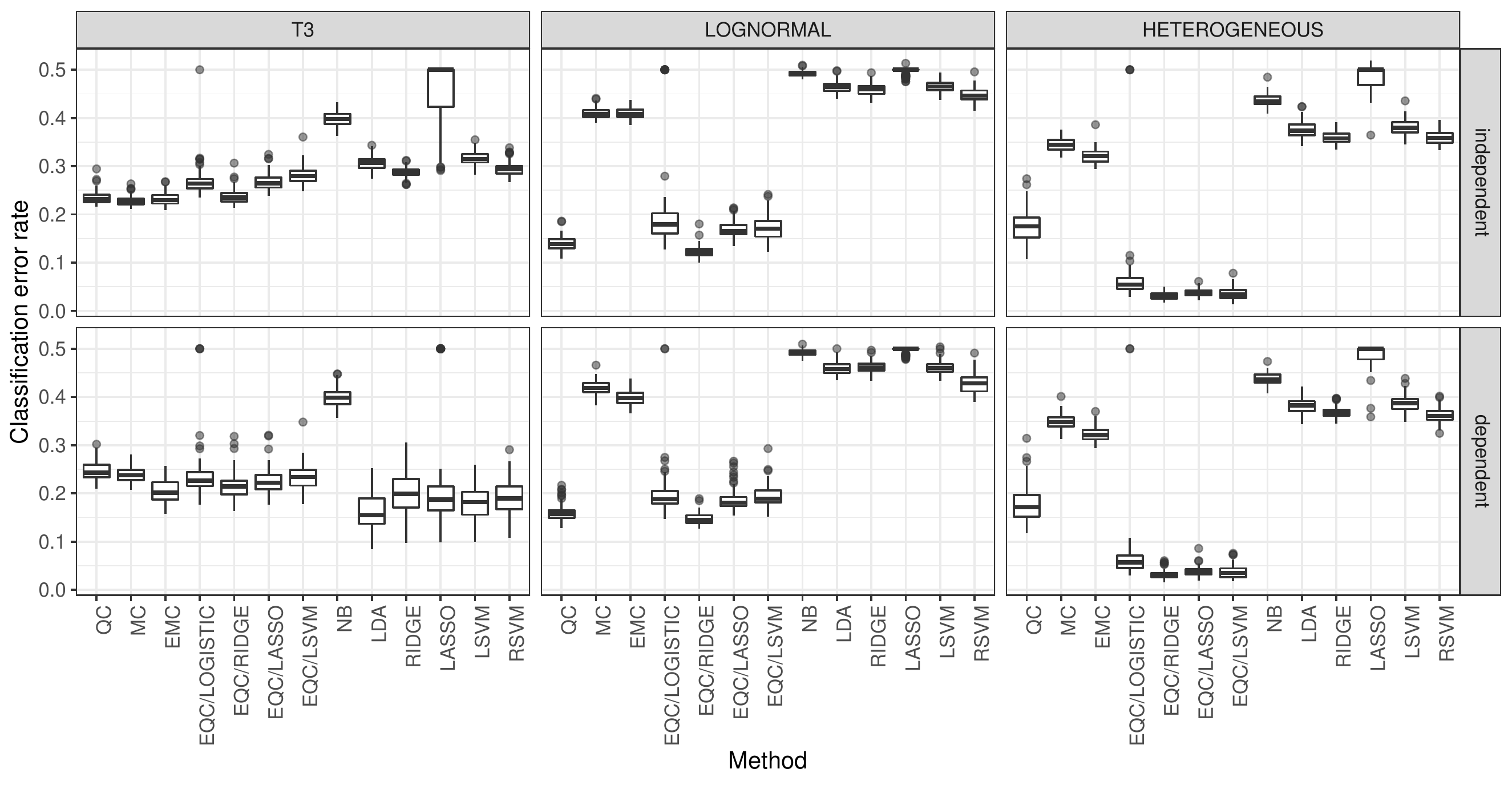}
\end{center}
\caption{
Comparison of test error rates with independent and correlated input variables
in the low dimensional scenario, $n=200, p=50$.
}
\label{fig:Err_n200p50cor}
\end{figure}

\subsection{Comparing EQC/RIDGE with QC for fixed \texorpdfstring{${\theta}$}{theta}}
\label{sec:thetaInfluence}

\ref{fig:theta} shows
the mean test rate of QC and EQC/RIDGE trained on a sample of size $n=100$ 
and evaluated over a grid for $\theta \in (0,1)$,
where 200 simulations for $10^4$ test samples for each parameter setting
and grid point were used.
The confidence limits are too narrow to show.
Looking along the first row of panels corresponding to the \textbf{T3} case,
the performance of EQC/RIDGE and QC is about the same for all $\theta$.
Since $\theta=0.5$ corresponds to the Bayes optimal median centroid \citep{HALL2009},
both QC and EQC/RIDGE provide optimal performance in the \textbf{T3} case
when $\theta=0.5$.
For the \textbf{LOGNORMAL} and \textbf{HETEROGENOUS} scenarios
EQC/RIDGE outperforms QC.
This figure demonstrates that the estimation
of a suitable $\theta$ is important in achieving a low test error rate.

\begin{figure}[H]
\begin{center}
\includegraphics[scale=0.65]{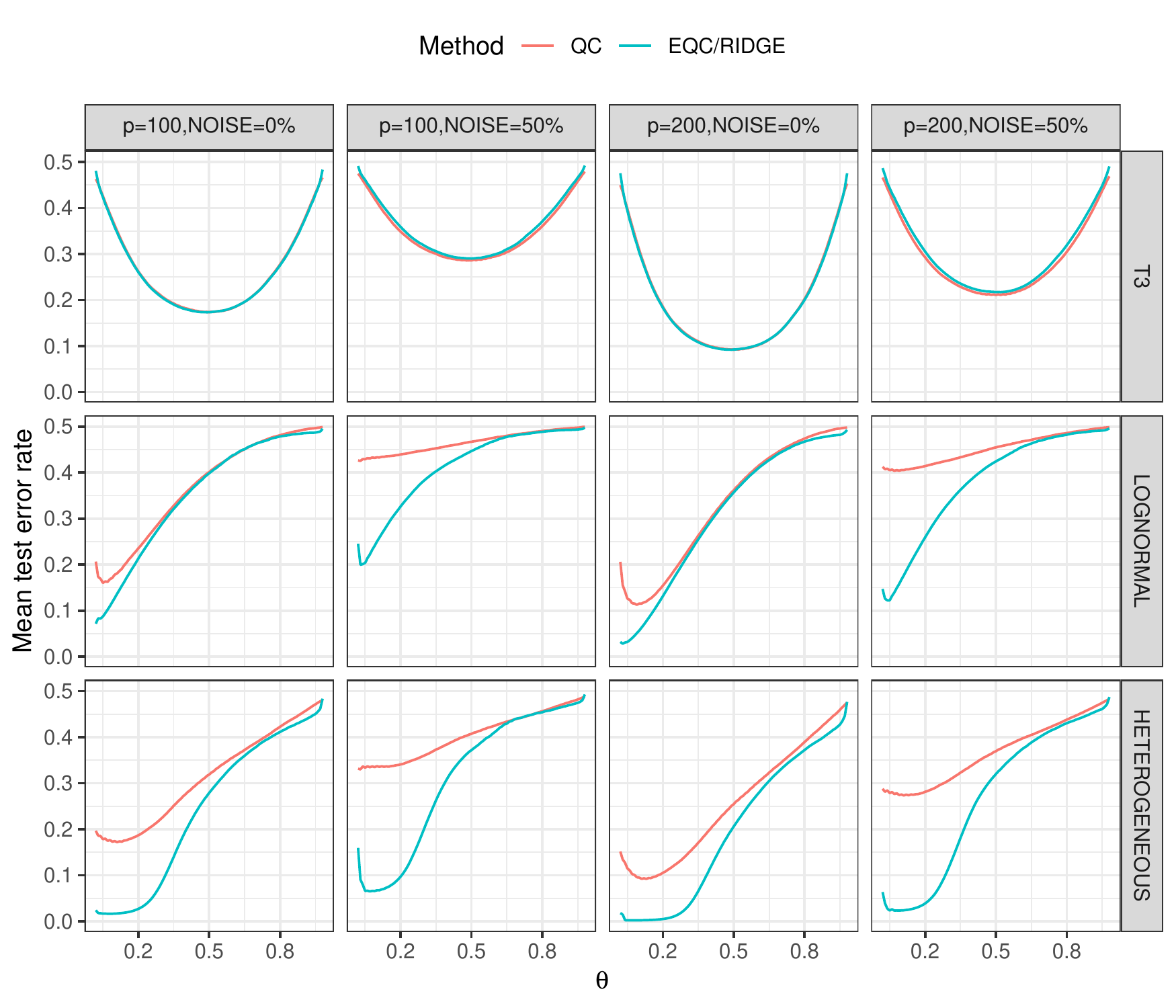}
\end{center}
\caption{Mean test error rates of the QC and  the EQC/RIDGE
against $\theta$ for fixed $n=100$, the three distributional scenarios,
the number of variables $p=100, 200$ and ${\rm NOISE} = 0\%, 50\%$.
}
\label{fig:theta}
\end{figure}

\section{Reuters-21578 text categorization}
\label{sec:application}

\subsection{Binary classification}
\label{sec:binaryReuter}
As in \citet{HALL2009} we used a subset of the Reuters-21578 text categorization
test collection \citep{Lewis1997, Sebastiani2002} to demonstrate
the usefulness of EQC and its improved performance over MC and QC.
The improved performance may be expected since this data set is high-dimensional,
sparse and the variables are highly skewed.

The subset contains two topics, ``acq'' and ``crude'', which can be found
from the R package \texttt{tm} \citep{tm}.
The subset has 70 observations (documents), where $n_1=50$ are of
the topic ``acq'' and $n_2=20$ are of the topic ``crude''.
The raw data set was preprocessed to first remove digits, punctuation marks,
extra white spaces, then convert to lower case, and
remove stop words and reduce to their stem.
It ended up with a $70\times 1517$ document-term matrix,
where a row represents a document and
a column represents a term, recording the frequency of a term.
A summary of the processed data set is shown in \ref{tab:ReuterSummary}.

The performance of a classifier was assessed by the mean
classification error rate estimated by 5 repetitions of
10-fold cross-validations with each fold containing $5$ documents of
the topic ``acq'' and $2$ documents of the topic ``crude''.

Since the performances of some classifiers such as the naive Bayes classifier
and the LDA could be much improved by using external feature selection strategies,
three external strategies for variable selection were investigated.
The first strategy was to use a subset of the data by removing low frequency
terms that appear in only one document, denoted by \texttt{removeLowFreq}.
This produced a $70\times 766$ document-term matrix.
The second and third strategies used Fisher's exact test to
select $L=50$ or $L=1000$ terms with the smallest p-values
within each fold of the cross-validation.

\begin{table}[H]
\caption{Summary of the Reuters-21578 subset.}
\begin{centering}
\label{tab:ReuterSummary}
\footnotesize
\begin{tabular}{ccccc}
\hline
\#Classes & \#Samples & \#Samples(acq) & \#Samples(crude) & \#Features\\
\hline
2 (acq vs crude) & 70 & 50 & 20 & 1517\\
\hline
\end{tabular}
\par\end{centering}
\end{table}

\ref{tab:Reuter} shows
the estimated error rates and their estimated standard errors for each classifier.
The second column indicates the situation where no external feature selection was used.
The four EQC methods, including the EMC,
performed the best even without any external feature selections,
followed by the QC and the MC.
It was found that most of the quantile-difference
transformed variables were constants, which can be removed.
This sparsity may explain the improved performance of the EQC family.

\begin{table}[H]
\centering
\caption{  Mean classification error rates,
and their standard errors in parentheses,
from 5 repetitions of 10-fold cross-validations for the
Reuters-21578 subset.}
\label{tab:Reuter}
\footnotesize
\begin{tabular}{lcccc}
\hline
\multirow{2}{*}{Method} & \multicolumn{4}{c}
{Classification Error Rates}\tabularnewline
\cline{2-5}
 & \multicolumn{1}{c}{Overall} & \multicolumn{1}{c}{RemoveLowFreq} &
 \multicolumn{1}{c}{$L=50$} & $L=1000$\tabularnewline
  \hline
QC & 0.069(0.013) & 0.06(0.012) & 0.049(0.01) & 0.063(0.012) \\
  MC & 0.06(0.012) & 0.063(0.014) & 0.054(0.012) & 0.063(0.014) \\
  EMC & \textbf{0.034}(0.01) & -- & -- & -- \\
  EQC/RIDGE & \textbf{0.034}(0.01) & -- &-- & -- \\
  EQC/LASSO & \textbf{0.037}(0.011) & -- & -- & -- \\
  EQC/LSVM & \textbf{0.034}(0.01) & -- & -- & -- \\
  NB & 0.714(0) & 0.714(0) & 0.117(0.015) & 0.134(0.015) \\
  LDA & 0.191(0.014) & 0.18(0.019) & 0.086(0.014) & 0.183(0.014) \\
  RIDGE & 0.203(0.012) & 0.186(0.012) & 0.097(0.013) & 0.203(0.012) \\
  LASSO & 0.051(0.013) & 0.049(0.013) & 0.066(0.014) & 0.049(0.013) \\
  LSVM & 0.109(0.013) & 0.1(0.013) & 0.091(0.014) & 0.1(0.013) \\
  RSVM & 0.217(0.012) & 0.203(0.016) & 0.097(0.014) & 0.191(0.018) \\
   \hline
\end{tabular}
\end{table}

\subsection{Multiclass classification}

To see how the EQC performs on the multiclass problem, a larger subset of Reuters-21578,
denoted by R8 \citep{phd-Ana-Cardoso-Cachopo} was tried.
This data set contains a training set and a test set that were obtained by
applying the modApte train/test split on the raw data \citep{Lewis1997}.
This resulted in retaining
8 classes with the highest number of positive training examples.
In order to classify those 8 classes,
the same preprocessing procedure as in \ref{sec:binaryReuter}
on the R8 data set was applied.
The terms were preprocessed to first remove digits, punctuation marks,
extra white spaces, then convert to lower case, and
remove stop words and reduce to their stem.
Terms that appeared in less than $0.5\%$ of documents were also removed, resulting in
a $5485\times 1367$ document-term matrix for training
and
a $2189\times 1367$ document-term matrix for testing.
The number of samples for each class is summarized in \ref{tab:R8Summary}.

Classifiers in the binary case were used but with the ridge logistic regression and
the LASSO logistic regression extended to the multinomial regressions,
and SVM extended to multi-class SVM by the one-against-one method \citep{tibs2009}.
\ref{tab:R8} shows the mean test error and the sensitivities of different classes.
The EQC still outperformed the other methods on the larger subset
while the EMC, the QC and the MC performed poorly this time.
With a much larger sample size,
the LSVM and the ridge multinomial regression were competitive
with EQC.

\begin{table}[H]
\caption{Summary of the Reuters-21578 subset R8 with 1367 features.}
\centering{}%
\label{tab:R8Summary}
\footnotesize
\begin{tabular}{lcc}
\hline
\multirow{2}{*}{Class} & \multicolumn{2}{c}{\#Samples}\\
\cline{2-3}
 & train  & test\\
\hline
acq & 1596 & 696 \\
  crude & 253 & 121 \\
  earn & 2840 & 1083 \\
  grain &  41 &  10 \\
  interest & 190 &  81 \\
  money-fx & 206 &  87 \\
  ship & 108 &  36 \\
  trade & 251 &  75 \\
\hline
Total & 5485 & 2189 \\
\end{tabular}
\end{table}

\begin{table}[H]
\centering
\caption{ Test error rate and sensitivities for
each multiclass classifier on the
Reuters-21578 subset R8.}
\label{tab:R8}
\footnotesize
\begin{tabular}{lccccccccc}
\hline
\multirow{2}{*}{Method} & \multirow{2}{*}{Test Error} & \multicolumn{8}{c}{Sensitivities}\\
\cline{3-10}
 &  & acq  & crude & earn & grain & interest & money-fx & ship & trade\\
\hline
  QC & 0.144 & 0.963 & 0.745 & 0.857 & 0.333 & 0.763 & 0.663 & 0.481 & 0.757 \\
  MC & 0.392 & 0.820 & 0.899 & 0.864 & 0.000 & 0.102 & 0.000 & 0.059 & 0.000 \\
  EMC & 0.309 & 0.750 & 0.925 & 0.875 & 1.000 & 0.700 & 0.316 & 0.091 & 0.908 \\
  EQC & \textbf{0.044} & 0.956 & 0.964 & 0.985 & 0.692 & 0.865 & 0.805 & 0.757 & 0.910 \\
  NB & 0.994 & 0.000 & 0.000 & 0.000 & 0.005 & 1.000 & 1.000 & 0.000 & 1.000 \\
  LDA & 0.109 & 0.896 & 0.883 & 0.906 & 1.000 & 0.732 & 0.709 & 0.759 & 0.906 \\
  RIDGE & 0.060 & 0.934 & 0.944 & 0.956 & 0.833 & 0.946 & 0.839 & 0.923 & 0.850 \\
  LASSO & 0.168 & 0.908 & 0.932 & 0.780 & 0.000 & 0.955 & 1.000 & 1.000 & 1.000 \\
  LSVM & 0.083 & 0.960 & 0.847 & 0.938 & 0.667 & 0.855 & 0.778 & 0.641 & 0.747 \\
  RSVM & 0.131 & 0.739 & 0.951 & 0.975 & 1.000 & 1.000 & 0.775 & 0.900 & 0.889 \\
\hline
\end{tabular}
\end{table}

\section{Discussion and conclusion}

The aim of the ensemble quantile classifier is to
derive a regularized weighted quantile-based classifier
that can best retain the advantage of QC on skewed inputs
and overcome the limitation of the QC with high-dimensional data
that includes noisy inputs.
The improvement using EQC has been demonstrated in simulation
experiments as well as with an application to text categorization.

We implemented the EQC methods in \cite{eqcGithub}, where
a vignette is available for
reproducing the simulations and
the Reuters text categorization application
in the paper.

\newpage
\section*{Acknowledgment}

The authors would like to thank the Associate Editor and the referees
for their insightful comments and suggestions,
which significantly improved the manuscript.
This research was supported by an NSERC Discovery Grant awarded to A. I. McLeod.
The simulations reported in this paper were made possible by the facilities of the Shared Hierarchical
Academic Research Computing Network (SHARCNET:www.sharcnet.ca) and Compute/Calcul Canada.

\appendix
\section*{Appendix}

\section{Relationship to asymmetric Laplace distribution}
\label{appendix:ASL}
A random variable $x$ is said to follow the asymmetric Laplace distribution, denoted as $x\sim\AL$,
if its probability density function
has the form,
\begin{equation}
\label{eq:ALpdf}
f(x)=\frac{\lambda}{\kappa+1/\kappa}
\begin{cases}
e^{\frac{\lambda}{k}(x-m)}, & \text{if }x<m\\
e^{-\lambda k(x-m)}, & \text{if }x\ge m
\end{cases},
\end{equation}
where $m\in \numset{R}$, $\lambda>0$ and $\kappa>0$ respectively are
the location, the scale and the skewness parameters.

Let $\pi_1$ and $\pi_2$ be the prior probabilities of $P_1$ and $P_2$.
If $P_1$ and $P_2$ consist of independent
asymmetric Laplace distribution with parameters $(\vec{m_1},\vec{\kappa},\vec{\lambda})$
and $(\vec{m_2},\vec{\kappa},\vec{\lambda})$,
then the Bayes decision boundary becomes $\set{\vec{x}}{s_{AL}(\vec{x})=0}$ with,
\begin{equation}
s_{AL}(\vec{x})=\log(\pi_2/\pi_1)+
\sum_{j=1}^p \lambda_{j} (\kappa_{j}+\kappa_{j}^{-1}) S_{al}(x_j,m_{1j},m_{2j},\kappa_{j},\lambda_{j}),
\label{eq:decisionBoundAL1}
\end{equation}
where for $m_1<m_2$,
\begin{align*}
S_{al}(x,m_{1},m_{2},\kappa,\lambda)=
\begin{cases}
-\frac{1}{\kappa^2+1} (m_2-m_1), & \text{if}\ x<m_{1}\\
x-\frac{\kappa^2}{\kappa^2+1} m_1
-\frac{1}{\kappa^2+1} m_2, & \text{if}\ m_{1} \le x <m_{2}\\
\frac{\kappa^2}{\kappa^2+1} (m_2-m_1), & \text{if}\ x\ge m_{2}
\end{cases}.
\end{align*}

Since $m$ is also the $[\kappa^2/(1+\kappa^2)]$-quantile of an asymmetric Laplace distribution,
if we let $\theta_j=\kappa_j^2/(1+\kappa_j^2)$, \ref{eq:decisionBoundAL1} will become the
$\mathsf{C}\bigl(\Qts(\vec{x}) \mid \beta_0,\vec{\beta}\bigr)$ of EQC in \ref{eq:EQCdecisionBound_linear} with
$\beta_0=\log(\pi_2/\pi_1)$ and
$\beta_j=\lambda_j\sqrt{[\theta_j(1-\theta_j)]^{-1}}$ for $j=1,\dotsc,p$.
If $x_j$'s are rescaled by its standard deviation
$\sqrt{1+\kappa^4}/(\lambda \kappa)$ first, then the
$\vec{\beta}$ will become
\begin{equation}
\beta_j=\frac{\sqrt{2}}{\theta_{j}(1-\theta_{j})}\sqrt{(\theta_{j}-\frac{1}{2})^2+\frac{1}{4}},
\text{ for } j=1,\dotsc,p.
\label{eq:ALw}
\end{equation}

Therefore, we can see that the decision boundary given by the EQC is
the Bayes decision boundary in this special case while QC cannot be
if $\vec{\kappa}$ is not homogeneous.

\section{Maximum likelihood estimation of multiclass EQC}
\label{appendix:MLEmulticlassEQC}
In this section, we formulate the log-likelihood function
of for the multiclass EQC in a matrix form
as well as its gradient vector and Hessian function.
This is useful for further investigation of theoretical properties and
the ease of computation.
At the end, we will show that the Hessian matrix is semi-negative-definite
so the log-likelihood function has a single, unique maximum.

Without loss of generality,
we let $\beta_{0,k}=0$ for all $k=1,\dotsc,K-1$
and disregard them.
Define
$\onevector_{K} = (1)_{1 \times K}$, $\onevector_{n} = (1)_{1 \times n}$,
$\mat{Y}=(y_{k,i})_{K \times n}$, where $y_{k,i}=1$ if $y_i=k$ and $0$ otherwise.
We also define,
\[
\mat{Q}_i =
\begin{pmatrix}
-\Qts^{(1,K)}(\vec{x}_i) \\
\vdots \\
-\Qts^{(K,K)}(\vec{x}_i)
\end{pmatrix},
\text{ for } i=1,\dotsc,n,
\]

\[
\mat{Q} =
\begin{pmatrix}
\mat{Q}_1 \\
\vdots \\
\mat{Q}_n
\end{pmatrix},
\]

\[
\vec{\beta} = (\beta_1,\dotsc,\beta_p)^{\transpose},
\]
and
\begin{align*}
\vec{C} &= \Bigl(C_{1}(\vec{\beta}),\dotsc,C_{n}(\vec{\beta})\Bigr)^{\transpose} \\
& =
\begin{pmatrix}
\onevector_{K} \exp[\mat{Q}_{1} \vec{\beta}]  \\
\vdots \\
\onevector_{K} \exp[\mat{Q}_{n} \vec{\beta}]
\end{pmatrix} \\
&= \mat{B}_{1} \exp[\mat{Q}\vec{\beta}],
\end{align*}
where $\exp[\cdot]$ is the entrywise exponential operation
and $\mat{B}_{1}=\Diag_{n}(\onevector_{K}, \dotsc, \onevector_{K})$ is a block diagonal matrix
consisting of $n$ repetitions of $\onevector_{K}$.

Then the log-likelihood function of EQC given $\vec{\theta}$ can be expressed,
\begin{align}
\loglikelihood(\vec{\beta}) &=
\Vectorize[\mat{Y}]^{\transpose} \mat{Q}\vec{\beta} -
\onevector_{n} \log[\vec{C}] \nonumber \\
&= \Vectorize[\mat{Y}]^{\transpose}\mat{Q} \vec{\beta} -
\onevector_{n} \log\Bigl[B_{1} \exp[\mat{Q} \vec{\beta}]\Bigr],
\label{eq:loglikMat}
\end{align}
where $\log[\cdot]$ is the entrywise natural logarithm operation
and $\Vectorize[\cdot]$ is a matrix vectorization operation which creates
a column vector by appending all columns of the matrix.

The gradient vector can be expressed,
\begin{equation}
\gradient \loglikelihood(\vec{\beta}) =
\Vectorize[\mat{Y}]^{\transpose}\mat{Q} -
\onevector_{n}
\Bigl[
\mat{A}
\entrydivide \mat{E}_{2}
\Bigr],
\label{eq:grloglikMat}
\end{equation}
where
\[
\mat{A} = \mat{B}_{1}\,\mat{U},
\]
\[
\mat{U} = \mat{Q} \entrydot \mat{E}_{1},
\]
$\entrydot$ stands for the entrywise multiplication
or the Hadamard product
and $\entrydivide$ stands for the entrywise division,
$\mat{E}_{1}$ is an $n K \times p$ matrix with $p$ repeated columns of
$\exp[\mat{Q} \vec{\beta}]$, which is,
\[
\mat{E}_{1}=\underbrace{
\begin{pmatrix}
\exp[\mat{Q} \vec{\beta}] &
\dotsc &
\exp[\mat{Q} \vec{\beta}]
\end{pmatrix}
}_{p \text{ repeated columns}},
\]
and $\mat{E}_{2}$ is an $n \times p$ matrix with $p$ repeated columns of
$\vec{C}=\mat{B}_{1} \exp[\mat{Q} \vec{\beta}]$, which is,
\[
\mat{E}_{2}=\underbrace{
\begin{pmatrix}
\vec{C} &
\dotsc &
\vec{C}
\end{pmatrix}
}_{p \text{ repeated columns}}.
\]
In particular, the $j$-th element of $\gradient \loglikelihood(\vec{\beta})$ is,
for $j=1,\dotsc,p$,
\[
\gradient_j \loglikelihood(\vec{\beta}) =
\Vectorize[\mat{Y}]^{\transpose}\mat{Q} \unitvector_{j}-
\onevector_{n}
\bigl[
\mat{A}_j
\entrydivide
C
\bigr],
\]
where
\[
\mat{A}_j = \mat{A} \unitvector_{j}=
\mat{B}_{1}\, \bigl[\mat{Q} \unitvector_{j} \entrydot \exp(\mat{Q} \vec{\beta})\bigr],
\]
$\unitvector_{j}$ is a unit column vector of length $p$ where the $j$-th element is $1$
and the other elements are $0$'s.

The $j$-th row of the Hessian matrix can be expressed as,
for $j=1,\dotsc,p$,
\begin{equation}
\gradient \bigl(\gradient_j \loglikelihood(\vec{\beta})\bigr) =
\mat{F}_{j} \mat{A} -
[\onevector_{n} \entrydivide \vec{C}^\transpose]\,\mat{B}_{1} \mat{G}_{j},
\label{eq:HessianloglikMatrow}
\end{equation}
where
\[
\mat{F}_{j} = \bigl[ \mat{A}_{j} \entrydivide [\vec{C} \entrydot \vec{C}] \bigr]^{\transpose} ,
\]
and
\[
\mat{G}_{j} =
\mat{U} \entrydot \mat{E}_{3,j},
\]
and
\[
\mat{E}_{3,j}=\underbrace{
\begin{pmatrix}
\mat{Q}\unitvector_{j} &
\dotsc &
\mat{Q}\unitvector_{j}
\end{pmatrix}
}_{p \text{ repeated columns}}.
\]
Then the Hessian matrix can be expressed,
\begin{equation}
\gradient^2 \loglikelihood(\vec{\beta}) =
\begin{pmatrix}
\mat{F}_{1}\\
\vdots \\
\mat{F}_{p}
\end{pmatrix}
\mat{A} -
\mat{B}_{2}
\begin{pmatrix}
\mat{G}_{1}\\
\vdots \\
\mat{G}_{p}
\end{pmatrix}
\label{eq:HessianloglikMat}
\end{equation}
where $\mat{B}_{2}=\Diag_{p}\bigl(
[\onevector_{n} \entrydivide \mat{C}^\transpose]\,\mat{B}_{1}, \dotsc,
[\onevector_{n} \entrydivide \mat{C}^\transpose]\,\mat{B}_{1}\bigl)$
is a block diagonal matrix consists of $p$ repetitions of
$[\onevector_{n} \entrydivide \mat{C}^\transpose]\,\mat{B}_{1}$.

If we let $\mat{W}_1=\Diag_{n K}\bigl(\exp(\mat{Q} \vec{\beta})\bigr)$,
$\mat{W}_2=\Diag_{n}(\onevector_{n} \entrydivide \mat{C}^\transpose)$,
and
\[
\mat{W}_{3}=\Diag_{n K}(
\underbrace{\frac{1}{C_1(\vec{\beta)}},\dotsc,\frac{1}{C_1(\vec{\beta)}}}_{K}
,\dotsc,
\underbrace{\frac{1}{C_n(\vec{\beta)}},\dotsc,\frac{1}{C_n(\vec{\beta)}}}_{K}
),
\]
then $\mat{A}$ can be expressed,
\[
\mat{A}=\mat{B}_{1} \mat{W}_{1} \mat{Q},
\]
and \ref{eq:grloglikMat} and \ref{eq:HessianloglikMat} can also be expressed,
\[
\gradient \loglikelihood(\vec{\beta}) =
\Vectorize[\mat{Y}]^{\transpose}\mat{Q} -
\onevector_{n}
\mat{W}_{2} \mat{A},
\]
and
\begin{align*}
\gradient^2 \loglikelihood(\vec{\beta}) & =
\mat{A}^\transpose \mat{W}_{2}^{2} \mat{A} -
\mat{Q}^\transpose \mat{W}_{3} \mat{W}_{1} \mat{Q} \\
& =
\mat{Q}^\transpose \mat{W}_{1}^\transpose
\bigl[
 \mat{B}_{1}^\transpose \mat{W}_{2}^{2} \mat{B}_{1} -
 \mat{W}_{1}^{-1} \mat{W}_{3}
\bigr]
\mat{W}_{1} \mat{Q}.
\end{align*}

In particular,
$\bigl[
 \mat{B}_{1}^\transpose \mat{W}_{2}^{2} \mat{B}_{1} -
 \mat{W}_{1}^{-1} \mat{W}_{3}
\bigr]$
is a negative semi-definite diagonal matrix as
its eigenvalues are all non-positive.
We can then conclude that the Hessian
of $\loglikelihood(\vec{\beta})$ is a negative
semi-definite matrix
and so is its L2 regularized version.

\section{Proof of the consistency of estimating EQC}
\label{appendix:proof}

Without loss of generality, we set $\beta_0=0$ and
disregard it
in the following discussions.

\begin{proof}[Proof of \ref{thm:consitencyClassifier}]
For abbreviation, denote $\vec{\eta}=(\vec{\theta},\vec{\beta})$.
From the continuity implied by \ref{lem:continuityQuanDis} later on,
we only need to show the following converges to zero,
\begin{equation}
\abs{\Psi(\tilde{\vec{\eta}})-\Psi(\hat{\vec{\eta}}_n)}\le
\abs{\Psi(\tilde{\vec{\eta}})-\Psi_n(\tilde{\vec{\eta}})}+
\abs{\Psi_n(\tilde{\vec{\eta}})-\Psi_n(\hat{\vec{\eta}}_n)}+
\abs{\Psi_n(\hat{\vec{\eta}_n})-\Psi(\hat{\vec{\eta}}_n)}.
\label{eq:upBound}
\end{equation}

By \ref{lem:convergeEmp} later on,
under \ref{assumption:continuousQuantile,assumption:boundedWeight},
$\forall \epsilon>0$,
\begin{equation}
\lim_{n\rightarrow\infty}\prob\{
\sup_{\vec{\eta}\in \numset{S}}
\abs{\Psi_n(\vec{\eta})-\Psi(\vec{\eta})}
> \epsilon\}=0,
\label{eq:convergeEmp}
\end{equation}
where $\numset{S} \subset \ointerval{0}{1}^p \times \numset{R}^{p}$.

So \ref{eq:convergeEmp} forces the first and the third term of
the right hand side of
\ref{eq:upBound} to converge to 0 in probability.

Consider the second term now.
By definitions of $\tilde{\vec{\eta}}$ and $\hat{\vec{\eta}}$,
\[
\Psi(\tilde{\vec{\eta}})\ge \Psi(\hat{\vec{\eta}}_n),\
\Psi_n(\tilde{\vec{\eta}})\le \Psi_n(\hat{\vec{\eta}}_n).
\]
So
\begin{align*}
\abs{\Psi_n(\tilde{\vec{\eta}})-\Psi_n(\hat{\vec{\eta}}_n)}=&
\Psi_n(\hat{\vec{\eta}}_n)-\Psi_n(\tilde{\vec{\eta}}) \\
= & [\Psi_n(\hat{\vec{\eta}}_n)-\Psi(\hat{\vec{\eta}}_n)]+
[\Psi(\hat{\vec{\eta}}_n)-\Psi_n(\tilde{\vec{\eta}})] \\
\le & [\Psi_n(\hat{\vec{\eta}}_n)-\Psi(\hat{\vec{\eta}}_n)]+
[\Psi(\tilde{\vec{\eta}})-\Psi_n(\tilde{\vec{\eta}})].
\end{align*}

Using \ref{eq:convergeEmp} again, then both
$\abs{\Psi_n(\hat{\vec{\eta}}_n)-\Psi(\hat{\vec{\eta}}_n)}$ and
$\abs{\Psi(\tilde{\vec{\eta}})-\Psi_n(\tilde{\vec{\eta}})}$
will converge to zero in probability.
This makes $\abs{\Psi_n(\tilde{\vec{\eta}})-\Psi_n(\hat{\vec{\eta}}_n)}$
converge to zero in probability.
Therefore, \ref{eq:upBound} converges to zero in probability.
That is,
\[
\lim_{n\rightarrow\infty}\prob\{
\abs{
\Psi(\tilde{\vec{\eta}})-
\Psi(\hat{\vec{\eta}}_n)
}>\epsilon
\}=0.
\]
\end{proof}

\begin{proof}[Proof of \ref{thm:consitencyEmpClassifier}]
For abbreviation, denote $\vec{\eta}=(\vec{\theta},\vec{\beta})$.
We will investigate
\begin{equation}
\abs{\Psi(\tilde{\vec{\eta}})-\Psi_n(\hat{\vec{\eta}}_n)}\le
\abs{\Psi(\tilde{\vec{\eta}})-\Psi(\hat{\vec{\eta}}_n)}+
\abs{\Psi(\hat{\vec{\eta}}_n)-\Psi_n(\hat{\vec{\eta}}_n)}.
\label{eq:upBound2}
\end{equation}

By \ref{thm:consitencyClassifier}, the first term of the right hand side above converges to zero in probability.
By \ref{lem:convergeEmp}, the second term of the right hand side above also converges to zero in probability.
Therefore, $\abs{\Psi(\tilde{\vec{\eta}})-\Psi_n(\hat{\vec{\eta}}_n)}$ converges to zero in probability and we complete the proof.
\end{proof}

\begin{lem}
\label{lem:continuityQuanDis}
Under the assumption that $\theta_{1}\le\theta_{2}$ and $q_{1}\le q_{2}$, or
$\theta_{1}\ge\theta_{2}$ and $q_{1}\ge q_{2}$,
the following inequality holds,
\[
\abs{\rho_{\theta_1}(x_j - q_{1}(\theta_1))-
\rho_{\theta_2}(x_j - q_{2}(\theta_2))}\le
\abs{x_j}\abs{\theta_2-\theta_1}+4\abs{q_{1}-q_{2}}
,\ j=1,\dotsc,p,
\]
which further implies the continuity of
$\mathsf{C}(\Qts(\vec{x}) \mid \vec{\beta})$ under \ref{assumption:continuousQuantile,assumption:boundedWeight},
and hence the continuity of $\Psi(\vec{\theta},\vec{\beta})$.
\end{lem}

\begin{proof}
The inequality follows directly from  the Lemma 3 in the supplementary material of \citet{Hennig2016}.

It implies that the quantile-based transformation $\Qts(\vec{x})$ is a continuous function of $\vec{\theta}$.
Furthermore, since empirical quantiles are strongly consistent,
$\vec{\beta}$ is bounded and
$\mathsf{C}(\vec{z} \mid \vec{\beta})$ is required to be
differentiable with respect to $\vec{z}$ and $\vec{\beta}$
by \ref{assumption:boundedWeight},
then $\mathsf{C}(\Qts(\vec{x}) \mid \vec{\beta})$ is bounded and a continuous function
of $\vec{\theta}$ and $\vec{\beta}$.
So the dominated convergence theorem still makes the integrals of
the differentiable transformation of $\mathsf{C}(\Qts(\vec{x}) \mid \vec{\beta})$
continuous with respect to $\vec{\theta}$ and $\vec{\beta}$.

\end{proof}

\begin{lem}
\label{lem:convergeEmp}
Under \ref{assumption:continuousQuantile,assumption:boundedWeight},
$\forall \epsilon>0$,
\[
\lim_{n\rightarrow\infty}\prob\{
\sup_{\vec{\eta}\in \numset{S}}
\abs{\Psi_n(\vec{\eta})-\Psi(\vec{\eta})}
> \epsilon\}=0,
\]
where $\vec{\eta}=(\vec{\theta},\vec{\beta})$ and
$\numset{S} \subset  \ointerval{0}{1}^p \times \numset{R}^{p}$,
\end{lem}

\begin{proof}
Assuming that the conclusion does not hold,
since $\vec{\eta}$ is bounded according to \ref{assumption:boundedWeight}, then
$\exists \epsilon >0$, $\delta >0$, there is a convergent subsequence
$\sequence[m]{\vec{\eta}^{*}_{m}}$ with limit $\vec{\eta}^*=\lim_{m\rightarrow\infty}\vec{\eta}^{*}_{m}$
such that for $m=1,\dotsc$,
\begin{equation}
\label{eq:oppResult}
\prob\{
\abs{\Psi_{m}(\vec{\eta}_{m}^{*})-\Psi(\vec{\eta}_{m}^{*})}
> \epsilon\}\ge \delta.
\end{equation}

Consider
\begin{equation}
\label{eq:upBoundsub}
\abs{\Psi_{m}(\vec{\eta}^{*}_{m})-\Psi(\vec{\eta}^{*}_{m})} \le
\abs{\Psi_{m}(\vec{\eta}^{*}_{m})-\Psi_{m}(\vec{\eta}^*)}+
\abs{\Psi_{m}(\vec{\eta}^*)-\Psi(\vec{\eta}^*)}+
\abs{\Psi(\vec{\eta}^*)-\Psi(\vec{\eta}^{*}_{m})}.
\end{equation}
Firstly, continuity of $\Psi(\vec{\eta})$ implies that
the third term of the right side
of \ref{eq:upBoundsub} converges to $0$ as $m\rightarrow\infty$.

Consider the second term,
we define a new $\Psi_n$ with the true quantiles below,
where the empirical decision rule $\mathsf{C}(\Qtsn(\vec{x}) \mid \vec{\beta})$
in \ref{eq:EBinomialDeviance}
is replaced by the population decision rule $\mathsf{C}(\Qts(\vec{x}) \mid \vec{\beta})$.
\[
\Psi_{m}^{*}(\vec{\eta})=-\frac{1}{m} \sum_{i=1}^{m}
\Bigl\{ (y_i-1) \mathsf{C}\bigl(\Qts(\vec{x}_i) \mid \vec{\beta}\bigr) -
\log (1+e^{ \mathsf{C}(\Qts(\vec{x}_i) \mid \vec{\beta}) })\Bigr\}.
\]

Consider
\[
\abs{\Psi_{m}(\vec{\eta}^*)-\Psi(\vec{\eta}^*)} \le
\abs{\Psi_{m}(\vec{\eta}^*)-\Psi^{*}_{m}(\vec{\eta}^*)}+
\abs{\Psi^{*}_{m}(\vec{\eta}^*)-\Psi(\vec{\eta}^*)}.
\]
Following the strong law of large numbers,
$\abs{\Psi^{*}_{m}(\vec{\eta}^*)-\Psi(\vec{\eta}^*)} \almscon 0$
as $m \rightarrow \infty$.

Since empirical quantiles are strongly consistent and
$\mathsf{C}(\vec{z} \mid \vec{\beta})$ is required to be
differentiable with respect to $\vec{z}$ and $\vec{\beta}$
by \ref{assumption:boundedWeight},
then $\abs{\mathsf{C}\bigl(\Qtsn[\vec{\theta}^{*}](\vec{x}) \mid \vec{\beta}^{*}\bigr)-
\mathsf{C}\bigl(\Qts[\vec{\theta}^{*}](\vec{x}) \mid \vec{\beta}^{*}\bigr)} \almscon 0$
as $m \rightarrow \infty$,
and hence
$\abs{\Psi_{m}(\vec{\eta}^*)-\Psi^{*}_{m}(\vec{\eta}^*)}\almscon 0$
as $m \rightarrow \infty$.

Now consider the first term of the right hand side of  \ref{eq:upBoundsub}.
Firstly, for $j=1,\dotsc,p$,
\[
\abs{\hat{q}_{k,j,m}(\theta_{j,m}^{*})-\hat{q}_{k,j,m}(\theta_{j}^*)}\le
\abs{\hat{q}_{k,j,m}(\theta_{j,m}^{*})-q_{k,j}(\theta_{j,m}^{*})}+
\abs{q_{k,j}(\theta_{j,m}^{*})-\hat{q}_{k,j}(\theta_{j}^*)}+
\abs{q_{k,j}(\theta_{j}^*)-\hat{q}_{k,j,m}(\theta_{j}^*)}.
\]
From Theorem 3 in \citet{Mason1982} and \ref{assumption:continuousQuantile}, all terms on the right side of
the above inequality converge to zero almost surely, and hence
$\abs{\hat{q}_{k,j,m}(\theta_{j,m}^{*})-\hat{q}_{k,j,m}(\theta_{j}^*)}
\almscon 0$
as $m \rightarrow \infty$ for $j=1,\dotsc,p$.
Thus $||\Qtsn[\vec{\theta}^{*}_{m}](\vec{x})-\Qtsn[\vec{\theta}^{*}](\vec{x}) ||
\almscon 0$ as $m \rightarrow \infty$, where $||\cdot||$ represents L2 norm.

Furthermore, since $\mathsf{C}(\vec{z} \mid \vec{\beta})$ is required to be
differentiable with respect to $\vec{z}$ and $\vec{\beta}$
by \ref{assumption:boundedWeight},
then $\abs{\mathsf{C}\bigl(\Qtsn[\vec{\theta}^{*}_{m}](\vec{x}) \mid \vec{\beta}^{*}_{m}\bigr)-
\mathsf{C}\bigl(\Qtsn[\vec{\theta}^{*}](\vec{x}) \mid \vec{\beta}\bigr)} \almscon 0$
as $m \rightarrow \infty$.
Thus the first term of the right hand side of
\ref{eq:upBoundsub} converges to zero almost surely,
$\abs{\Psi_{m}(\vec{\eta}_{m}^{*})-\Psi_{m}(\vec{\eta}^{*})} \almscon 0$ as $m \rightarrow \infty$.

To sum up, $\abs{\Psi^{*}_{m}(\vec{\eta}_{n_i})-\Psi(\vec{\eta}^{*}_{m})} \almscon 0$, which is contradictory to
\ref{eq:oppResult} and hence we conclude that,
under \ref{assumption:continuousQuantile,assumption:boundedWeight},
$\forall \epsilon>0$,
\[
\lim_{n\rightarrow\infty}\prob\{
\sup_{\vec{\eta} \in \numset{S}}
\abs{\Psi_n(\vec{\eta})-\Psi(\vec{\eta})}
> \epsilon\}=0,
\]
where $\numset{S} \subset \ointerval{0}{1}^p \times \numset{R}^{p}$.
\end{proof}

\newpage
\section{Misclassification rates}
\label{appendix:otherDetail}

In \ref{sec:simInterpret},
we presented boxplots of the test misclassification error rates
for each classifier.
Their averages and standard deviations are tabulated in
\crefrange{tab:simTabT3}{tab:simTabheterogeneousdep},
for the T3, LOGNORMAL and HETEROGENEOUS distribution cases
with independent or dependent variables.

\begin{table}[H]
\centering
\caption{ Simulation study:
the mean test classification error rates,
and their standard errors in parentheses,
of each method for the  \textbf{independent T3} scenario.
All numbers are in percentages and rounded to one digit.
The third line indicates the percentage of irrelevant variables within $p$ variables.}
\label{tab:simTabT3}
\scriptsize
\begin{tabular}{llllllllll}
\hline
\multicolumn{10}{c}{$n=100$} \\
 \hline & \multicolumn{3}{c}{$p=50$} & \multicolumn{3}{c}{$p=100$} & \multicolumn{3}{c}{$p=200$}\\
 \hline
   & $0\%$ & $50\%$ & $90\%$ & $0\%$ & $50\%$ & $90\%$  & $0\%$ & $50\%$ & $90\%$\\
  QC & 27.1(0.3) & 35.9(0.4) & 47.2(0.2) & 18.5(0.2) & 30.4(0.3) & 46(0.2) & 10(0.1) & 22.3(0.3) & 43.9(0.3) \\
  MC & 25.5(0.1) & 33.9(0.2) & 46(0.2) & 17.6(0.1) & 28.5(0.2) & 44.4(0.1) & 9.1(0.1) & 20.8(0.1) & 42.2(0.1) \\
  EMC & 26.1(0.2) & 35(0.2) & 46.6(0.2) & 18.2(0.2) & 29.7(0.2) & 45.2(0.2) & 9.5(0.1) & 21.9(0.2) & 43(0.2) \\
  EQC/LOGISTIC & 32.1(0.4) & 40.3(0.4) & 47.9(0.2) & 24.7(0.3) & 35.4(0.3) & 47.6(0.2) & - & - & - \\
  EQC/RIDGE & 28.4(0.4) & 37.9(0.5) & 48(0.2) & 19.4(0.3) & 32(0.4) & 47.1(0.2) & 10.6(0.2) & 23.9(0.4) & 44.6(0.2) \\
  EQC/LASSO & 33.5(0.5) & 40.2(0.4) & 47.8(0.2) & 28.6(0.3) & 37.2(0.3) & 47(0.2) & 25.8(0.4) & 32.7(0.4) & 46(0.2) \\
  EQC/LSVM & 33.5(0.4) & 40.5(0.4) & 48.4(0.2) & 24.5(0.3) & 35.9(0.3) & 47.6(0.2) & 14.1(0.3) & 26.8(0.3) & 45.4(0.2) \\
  NB & 41.5(0.1) & 43.8(0.2) & 48.2(0.1) & 39.8(0.1) & 42.3(0.1) & 47.6(0.1) & 37.8(0.1) & 40.5(0.1) & 46.7(0.1) \\
  LDA & 35.8(0.2) & 41.3(0.2) & 48.2(0.2) & 44.9(0.3) & 47.1(0.2) & 49.2(0.1) & 31.4(0.3) & 38.3(0.2) & 47.4(0.1) \\
  RIDGE & 31(0.2) & 38.4(0.2) & 47.3(0.2) & 25.3(0.1) & 34.4(0.2) & 46.3(0.1) & 18.4(0.1) & 28.9(0.2) & 44.8(0.1) \\
  LASSO & 45.5(0.5) & 47.7(0.4) & 49(0.2) & 44.3(0.6) & 46.4(0.4) & 49(0.2) & 42.5(0.7) & 45.5(0.5) & 49.2(0.2) \\
  LSVM & 36.2(0.3) & 41.9(0.2) & 48.1(0.2) & 29.8(0.2) & 38.2(0.2) & 47.3(0.1) & 21.3(0.1) & 32(0.2) & 45.9(0.1) \\
  RSVM & 32.1(0.2) & 39.2(0.2) & 47.6(0.2) & 26.3(0.2) & 35.3(0.2) & 46.6(0.1) & 18.7(0.2) & 29.6(0.2) & 45.2(0.1) \\
  \hline
\multicolumn{10}{c}{$n=200$} \\
 \hline & \multicolumn{3}{c}{$p=50$} & \multicolumn{3}{c}{$p=100$} & \multicolumn{3}{c}{$p=200$}\\
 \hline
   & $0\%$ & $50\%$ & $90\%$ & $0\%$ & $50\%$ & $90\%$  & $0\%$ & $50\%$ & $90\%$\\
 QC & 23.5(0.1) & 32.6(0.2) & 45.9(0.2) & 15.1(0.1) & 25.8(0.2) & 43.7(0.2) & 7.4(0.1) & 17.8(0.1) & 41.1(0.2) \\
  MC & 22.8(0.1) & 31.6(0.1) & 44.5(0.1) & 14.5(0.1) & 25(0.1) & 42.5(0.1) & 6.8(0) & 17.1(0.1) & 39.5(0.1) \\
  EMC & 23.2(0.1) & 32.3(0.1) & 45(0.1) & 14.8(0.1) & 25.7(0.2) & 43.3(0.1) & 7.3(0.1) & 18.1(0.1) & 40.5(0.1) \\
  EQC/LOGISTIC & 26.8(0.3) & 35.1(0.2) & 46.9(0.2) & 20.6(0.2) & 31.9(0.3) & 46(0.2) & 13.3(0.1) & 24.7(0.2) & 43.9(0.2) \\
  EQC/RIDGE & 23.7(0.2) & 33.1(0.2) & 46.6(0.2) & 15.5(0.2) & 26.5(0.2) & 44.9(0.2) & 7.7(0.1) & 18.7(0.2) & 42.2(0.3) \\
  EQC/LASSO & 26.8(0.2) & 34.9(0.3) & 46.3(0.2) & 22.3(0.2) & 30.9(0.3) & 45.2(0.2) & 18.8(0.2) & 26.3(0.2) & 43.3(0.2) \\
  EQC/LSVM & 28.1(0.2) & 35.8(0.3) & 47.1(0.2) & 22.1(0.2) & 32.6(0.3) & 46.5(0.2) & 11.9(0.1) & 24.1(0.2) & 44.1(0.2) \\
  NB & 39.8(0.1) & 42.8(0.2) & 47.5(0.1) & 37.9(0.1) & 40.5(0.1) & 46.7(0.1) & 35.3(0.1) & 38.7(0.1) & 45.6(0.1) \\
  LDA & 30.6(0.1) & 37.8(0.2) & 46.9(0.1) & 28.5(0.2) & 36.3(0.2) & 46.3(0.1) & 43.8(0.3) & 46.7(0.3) & 49(0.1) \\
  RIDGE & 28.8(0.1) & 36.3(0.2) & 46.4(0.1) & 22.7(0.1) & 31.7(0.1) & 44.9(0.1) & 15.7(0.1) & 26(0.1) & 43(0.1) \\
  LASSO & 44.9(0.6) & 46.8(0.4) & 48.7(0.2) & 41.1(0.9) & 45.8(0.5) & 48.3(0.3) & 32.9(1) & 44.1(0.5) & 47.7(0.3) \\
  LSVM & 31.6(0.1) & 38.3(0.2) & 47(0.1) & 28.1(0.2) & 37.1(0.2) & 46.6(0.1) & 19.6(0.1) & 30.5(0.1) & 45(0.1) \\
  RSVM & 29.4(0.1) & 37(0.2) & 46.7(0.1) & 22.8(0.1) & 32.2(0.1) & 45.2(0.1) & 15.2(0.1) & 26.7(0.1) & 43.7(0.1) \\
   \hline
\end{tabular}
\end{table}

\begin{table}[H]
\centering
\caption{ Simulation study:
the mean test classification error rates,
and their standard errors in parentheses,
of each method for the  \textbf{dependent T3} scenario.
All numbers are in percentages and rounded to one digit.
The third line indicates the percentage of irrelevant variables within $p$ variables.}
\label{tab:simTabt3dep}
\scriptsize
\begin{tabular}{llllllllll}
\hline
\multicolumn{10}{c}{$n=100$} \\
 \hline & \multicolumn{3}{c}{$p=50$} & \multicolumn{3}{c}{$p=100$} & \multicolumn{3}{c}{$p=200$}\\
 \hline
   & $0\%$ & $50\%$ & $90\%$ & $0\%$ & $50\%$ & $90\%$  & $0\%$ & $50\%$ & $90\%$\\
QC & 28.9(0.3) & 36.9(0.3) & 46.9(0.2) & 20.8(0.2) & 30.9(0.3) & 45.6(0.2) & 11.8(0.2) & 24(0.3) & 43.7(0.2) \\
  MC & 27.3(0.2) & 35.1(0.2) & 45.8(0.1) & 19.9(0.2) & 29.4(0.2) & 44.2(0.1) & 11(0.1) & 22.5(0.1) & 42.3(0.1) \\
  EMC & 25.1(0.3) & 34.6(0.3) & 46(0.2) & 18.2(0.2) & 29.1(0.2) & 44.6(0.1) & 10.2(0.1) & 22.6(0.2) & 42.9(0.1) \\
  EQC/LOGISTIC & 28.3(0.4) & 38.1(0.4) & 47.4(0.2) & 21.5(0.3) & 32.4(0.3) & 47.1(0.2) & - & - & - \\
  EQC/RIDGE & 26.5(0.4) & 36(0.4) & 47.3(0.2) & 19.4(0.3) & 31(0.3) & 46.5(0.2) & 10.8(0.2) & 24.2(0.3) & 44.6(0.2) \\
  EQC/LASSO & 29.2(0.5) & 37.5(0.5) & 47.3(0.2) & 26.2(0.3) & 33.6(0.4) & 46.9(0.2) & 24.9(0.3) & 32.3(0.4) & 45.9(0.3) \\
  EQC/LSVM & 28.5(0.3) & 37.6(0.4) & 47.7(0.2) & 20.8(0.3) & 32.9(0.3) & 46.9(0.2) & 12(0.2) & 25.4(0.3) & 45.4(0.2) \\
  NB & 41.4(0.2) & 43.9(0.2) & 48(0.1) & 39.6(0.1) & 42.5(0.1) & 47.6(0.1) & 37.6(0.1) & 40.4(0.1) & 46.7(0.1) \\
  LDA & 21.7(0.5) & 29.5(0.6) & 44.1(0.5) & 38.6(0.5) & 42.4(0.4) & 48.3(0.2) & 31.8(0.3) & 40.8(0.2) & 47.9(0.1) \\
  RIDGE & 28.4(0.4) & 38(0.3) & 47.2(0.1) & 23.3(0.3) & 34.3(0.2) & 46.2(0.1) & 17(0.2) & 29.1(0.2) & 44.7(0.1) \\
  LASSO & 43.7(0.9) & 46.4(0.7) & 49.3(0.3) & 43.6(0.7) & 46.9(0.4) & 49.4(0.2) & 44.3(0.7) & 46.1(0.5) & 49.5(0.1) \\
  LSVM & 23.2(0.5) & 32.6(0.5) & 45.6(0.4) & 21.5(0.3) & 33.2(0.3) & 46.2(0.2) & 17.2(0.2) & 29.3(0.2) & 45.1(0.1) \\
  RSVM & 26.1(0.4) & 36(0.4) & 46.9(0.1) & 22.4(0.3) & 33.4(0.2) & 46.2(0.1) & 17(0.2) & 29(0.2) & 44.9(0.1) \\
  \hline
\multicolumn{10}{c}{$n=200$} \\
 \hline & \multicolumn{3}{c}{$p=50$} & \multicolumn{3}{c}{$p=100$} & \multicolumn{3}{c}{$p=200$}\\
 \hline
   & $0\%$ & $50\%$ & $90\%$ & $0\%$ & $50\%$ & $90\%$  & $0\%$ & $50\%$ & $90\%$\\
  QC & 24.6(0.2) & 33.6(0.2) & 45.9(0.2) & 16.7(0.2) & 26.8(0.2) & 43.8(0.2) & 8.4(0.1) & 18.9(0.2) & 41(0.2) \\
  MC & 23.9(0.2) & 32.5(0.2) & 44.6(0.1) & 15.8(0.1) & 25.9(0.1) & 42.6(0.1) & 7.8(0.1) & 18.2(0.1) & 39.7(0.1) \\
  EMC & 20.5(0.2) & 30.2(0.3) & 44(0.2) & 13.4(0.1) & 24.8(0.2) & 42.9(0.2) & 6.6(0.1) & 17.4(0.2) & 40.3(0.1) \\
  EQC/LOGISTIC & 23.6(0.5) & 31.3(0.3) & 45.5(0.3) & 15.8(0.2) & 27.9(0.3) & 45(0.3) & 9.8(0.2) & 20.4(0.3) & 43.2(0.2) \\
  EQC/RIDGE & 21.5(0.3) & 31.1(0.3) & 45.6(0.3) & 14.2(0.2) & 26(0.3) & 44(0.2) & 7.1(0.1) & 18(0.2) & 41.9(0.3) \\
  EQC/LASSO & 22.4(0.3) & 31.1(0.3) & 45.4(0.3) & 18(0.2) & 26.4(0.3) & 43.3(0.3) & 16(0.2) & 22.5(0.2) & 41.6(0.4) \\
  EQC/LSVM & 23.4(0.3) & 32(0.3) & 45.7(0.3) & 16.7(0.2) & 28.8(0.3) & 44.9(0.2) & 8.6(0.1) & 20.4(0.2) & 43.1(0.2) \\
  NB & 39.8(0.2) & 42.4(0.2) & 47.6(0.1) & 37.3(0.1) & 40.5(0.1) & 46.9(0.1) & 35.1(0.1) & 38.4(0.1) & 45.7(0.1) \\
  LDA & 16.1(0.4) & 24(0.5) & 42.1(0.5) & 14.1(0.3) & 22.9(0.4) & 41(0.5) & 39.4(0.4) & 41.9(0.3) & 47.9(0.2) \\
  RIDGE & 19.9(0.5) & 31.8(0.5) & 46.3(0.1) & 15.5(0.3) & 28.1(0.3) & 45.1(0.1) & 10.8(0.1) & 23.1(0.2) & 43.2(0.1) \\
  LASSO & 21.9(1.1) & 30.8(1.2) & 48.9(0.4) & 22.5(1.1) & 33(1.3) & 48(0.5) & 23.9(1) & 39.1(1.1) & 47.9(0.3) \\
  LSVM & 18(0.4) & 24.9(0.4) & 40.9(0.6) & 15.5(0.2) & 26.3(0.4) & 43.1(0.3) & 11.4(0.1) & 23.6(0.2) & 43.6(0.2) \\
  RSVM & 19.2(0.4) & 29.3(0.4) & 44.9(0.3) & 15.3(0.3) & 26.9(0.3) & 44.6(0.2) & 11(0.2) & 22.8(0.2) & 43.1(0.1) \\
   \hline
\end{tabular}
\end{table}

\begin{table}[H]
\centering
\caption{ Simulation study:
the mean test classification error rates,
and their standard errors in parentheses,
of each method for the  \textbf{independent LOGNORMAL} scenario.
All numbers are in percentages and rounded to one digit.
The third line indicates the percentage of irrelevant variables within $p$ variables.}
\label{tab:simTabLOGNORMAL}
\scriptsize
\begin{tabular}{llllllllll}
\hline
\multicolumn{10}{c}{$n=100$} \\
 \hline & \multicolumn{3}{c}{$p=50$} & \multicolumn{3}{c}{$p=100$} & \multicolumn{3}{c}{$p=200$}\\
 \hline
& $0\%$ & $50\%$ & $90\%$ & $0\%$ & $50\%$ & $90\%$  & $0\%$ & $50\%$ & $90\%$\\
QC & 23.3(0.5) & 45.7(0.3) & 49.5(0.1) & 17.9(0.5) & 44.6(0.2) & 49.4(0.1) & 12.5(0.3) & 42.4(0.3) & 49(0.1) \\
  MC & 43(0.2) & 47.6(0.1) & 49.5(0.1) & 40(0.2) & 46.7(0.1) & 49.4(0.1) & 35.8(0.2) & 45.2(0.1) & 49.2(0.1) \\
  EMC & 43.4(0.2) & 46.3(0.2) & 49(0.1) & 40.1(0.2) & 44.9(0.2) & 49(0.1) & 35.8(0.2) & 42.4(0.1) & 48.4(0.1) \\
  EQC/LOGISTIC & 25.5(0.7) & 39.4(0.7) & 48.6(0.2) & 15.2(0.3) & 28.2(0.6) & 46.9(0.3) & - & - & - \\
  EQC/RIDGE & 15.3(0.3) & 28.2(0.4) & 47(0.3) & 8.1(0.2) & 20.8(0.3) & 45.8(0.3) & 2.8(0.1) & 12.3(0.3) & 41.9(0.4) \\
  EQC/LASSO & 24(0.4) & 33.4(0.6) & 47.3(0.3) & 20.8(0.4) & 28.2(0.6) & 45.7(0.4) & 20.3(0.3) & 24(0.4) & 42.2(0.5) \\
  EQC/LSVM & 22.5(0.4) & 34(0.5) & 47.4(0.2) & 12.1(0.3) & 26.2(0.4) & 46.7(0.3) & 4.2(0.1) & 15(0.3) & 43.3(0.4) \\
  NB & 49.3(0.1) & 49.4(0.1) & 49.7(0.1) & 49.4(0.1) & 49.4(0) & 49.7(0.1) & 49.3(0.1) & 49.3(0.1) & 49.6(0.1) \\
  LDA & 47.8(0.1) & 48.7(0.1) & 49.7(0.1) & 49.3(0.1) & 49.6(0.1) & 50(0.1) & 47(0.2) & 48.3(0.1) & 49.6(0.1) \\
  RIDGE & 46.8(0.1) & 48.2(0.1) & 49.6(0.1) & 45.6(0.1) & 47.6(0.1) & 49.5(0.1) & 43.9(0.1) & 46.6(0.1) & 49.4(0.1) \\
  LASSO & 49.7(0.1) & 49.7(0.1) & 50(0) & 49.5(0.1) & 49.9(0.1) & 49.9(0) & 49.4(0.1) & 49.8(0.1) & 50(0) \\
  LSVM & 47.9(0.1) & 48.8(0.1) & 49.7(0.1) & 47.2(0.1) & 48.4(0.1) & 49.6(0.1) & 45.1(0.1) & 47.2(0.1) & 49.5(0.1) \\
  RSVM & 46.3(0.1) & 48.2(0.1) & 49.6(0.1) & 45.8(0.1) & 47.9(0.1) & 49.6(0.1) & 44.3(0.1) & 46.9(0.1) & 49.5(0.1) \\
  \hline
\multicolumn{10}{c}{$n=200$} \\
 \hline & \multicolumn{3}{c}{$p=50$} & \multicolumn{3}{c}{$p=100$} & \multicolumn{3}{c}{$p=200$}\\
 \hline
 & $0\%$ & $50\%$ & $90\%$ & $0\%$ & $50\%$ & $90\%$  & $0\%$ & $50\%$ & $90\%$\\
QC & 13.9(0.1) & 41.7(0.4) & 49.3(0.1) & 7.2(0.1) & 41.3(0.3) & 49.1(0.1) & 2.7(0.1) & 38.2(0.2) & 48.7(0.1) \\
  MC & 41(0.1) & 46.8(0.1) & 49.5(0.1) & 37.4(0.1) & 45.5(0.1) & 49.2(0.1) & 32.5(0.1) & 43.7(0.1) & 48.9(0.1) \\
  EMC & 41(0.1) & 45.1(0.1) & 48.9(0.1) & 37.9(0.2) & 43.5(0.2) & 48.5(0.1) & 32.8(0.2) & 40.6(0.1) & 47.9(0.1) \\
  EQC/LOGISTIC & 20.2(0.9) & 28.2(0.4) & 44.9(0.3) & 10.2(0.2) & 24.8(0.6) & 46.2(0.3) & 5.9(0.1) & 13.4(0.2) & 41.1(0.4) \\
  EQC/RIDGE & 12.3(0.1) & 23.3(0.2) & 44.4(0.3) & 5.8(0.1) & 14.9(0.2) & 41.6(0.3) & 1.9(0) & 7.8(0.1) & 37(0.3) \\
  EQC/LASSO & 16.9(0.2) & 24.6(0.2) & 44.5(0.3) & 12.8(0.2) & 19(0.2) & 42.1(0.4) & 11.7(0.2) & 14.7(0.2) & 35.8(0.4) \\
  EQC/LSVM & 17.1(0.2) & 27.5(0.3) & 45.4(0.3) & 9.3(0.2) & 22.2(0.3) & 44.4(0.3) & 3.1(0.1) & 11.1(0.2) & 39.7(0.3) \\
  NB & 49.2(0.1) & 49.4(0.1) & 49.6(0.1) & 49.1(0) & 49.2(0.1) & 49.7(0.1) & 49.1(0) & 49.3(0.1) & 49.6(0) \\
  LDA & 46.4(0.1) & 48(0.1) & 49.5(0.1) & 45.9(0.1) & 48(0.1) & 49.6(0.1) & 49.1(0.1) & 49.6(0.1) & 49.9(0.1) \\
  RIDGE & 45.8(0.1) & 47.8(0.1) & 49.5(0.1) & 44.3(0.1) & 47.1(0.1) & 49.5(0.1) & 42.6(0.1) & 45.9(0.1) & 49.2(0.1) \\
  LASSO & 49.7(0.1) & 49.8(0) & 50(0) & 49.7(0.1) & 49.8(0.1) & 50(0) & 49.4(0.1) & 49.8(0.1) & 50(0) \\
  LSVM & 46.6(0.1) & 48(0.1) & 49.6(0.1) & 46(0.1) & 48.2(0.1) & 49.6(0.1) & 44.8(0.1) & 47.2(0.1) & 49.3(0.1) \\
  RSVM & 44.8(0.1) & 47.5(0.1) & 49.5(0.1) & 43.8(0.1) & 47(0.1) & 49.4(0.1) & 42.3(0.1) & 46.1(0.1) & 49.3(0.1) \\
   \hline
\end{tabular}
\end{table}

\begin{table}[H]
\centering
\caption{ Simulation study:
the mean test classification error rates,
and their standard errors in parentheses,
of each method for the  \textbf{dependent LOGNORMAL} scenario.
All numbers are in percentages and rounded to one digit.
The third line indicates the percentage of irrelevant variables within $p$ variables.}
\label{tab:simTablogNormaldep}
\scriptsize
\begin{tabular}{llllllllll}
\hline
\multicolumn{10}{c}{$n=100$} \\
 \hline & \multicolumn{3}{c}{$p=50$} & \multicolumn{3}{c}{$p=100$} & \multicolumn{3}{c}{$p=200$}\\
 \hline
& $0\%$ & $50\%$ & $90\%$ & $0\%$ & $50\%$ & $90\%$  & $0\%$ & $50\%$ & $90\%$\\
QC & 24.4(0.5) & 46.1(0.3) & 49.6(0.1) & 18.9(0.5) & 45(0.3) & 49.4(0.1) & 14(0.4) & 42.4(0.3) & 49.1(0.1) \\
  MC & 43.8(0.2) & 47.7(0.1) & 49.6(0.1) & 41.6(0.2) & 46.7(0.1) & 49.5(0.1) & 37.9(0.1) & 45.6(0.1) & 49.2(0.1) \\
  EMC & 42.7(0.2) & 46.1(0.1) & 49.3(0.1) & 40.3(0.2) & 44.7(0.1) & 48.9(0.1) & 36.8(0.1) & 42.9(0.1) & 48.3(0.1) \\
  EQC/LOGISTIC & 28.6(0.8) & 41.6(0.6) & 49(0.2) & 17.6(0.4) & 29.5(0.4) & 47.6(0.2) & - & - & - \\
  EQC/RIDGE & 17.6(0.2) & 30.9(0.5) & 47.3(0.2) & 10.2(0.2) & 22.3(0.3) & 45.9(0.3) & 4(0.1) & 14.3(0.2) & 42.7(0.4) \\
  EQC/LASSO & 25(0.4) & 35.4(0.6) & 47.8(0.3) & 21.8(0.4) & 29.6(0.5) & 46.5(0.3) & 21.1(0.3) & 24.9(0.4) & 42.9(0.5) \\
  EQC/LSVM & 24.7(0.4) & 35.7(0.5) & 48(0.2) & 15.1(0.3) & 28.2(0.4) & 47(0.2) & 5.9(0.1) & 17.4(0.2) & 44.3(0.3) \\
  NB & 49.3(0.1) & 49.4(0.1) & 49.7(0.1) & 49.3(0.1) & 49.5(0.1) & 49.8(0.1) & 49.3(0.1) & 49.5(0.1) & 49.7(0.1) \\
  LDA & 47.4(0.1) & 48.6(0.1) & 49.7(0.1) & 49.3(0.1) & 49.3(0.1) & 50.1(0.1) & 46.7(0.1) & 48.4(0.1) & 49.6(0.1) \\
  RIDGE & 46.9(0.1) & 48.4(0.1) & 49.6(0.1) & 46(0.1) & 47.9(0.1) & 49.6(0.1) & 44.7(0.1) & 47.1(0.1) & 49.4(0.1) \\
  LASSO & 49.7(0.1) & 49.9(0) & 50(0) & 49.6(0.1) & 49.8(0.1) & 50(0) & 49.5(0.1) & 49.7(0.1) & 50(0) \\
  LSVM & 47.3(0.2) & 48.6(0.1) & 49.8(0.1) & 46.7(0.1) & 48.4(0.1) & 49.8(0.1) & 45.2(0.1) & 47.5(0.1) & 49.5(0.1) \\
  RSVM & 45.7(0.2) & 48.1(0.1) & 49.6(0.1) & 45.7(0.1) & 47.9(0.1) & 49.7(0.1) & 44.5(0.1) & 47.2(0.1) & 49.4(0.1) \\
  \hline
\multicolumn{10}{c}{$n=200$} \\
 \hline & \multicolumn{3}{c}{$p=50$} & \multicolumn{3}{c}{$p=100$} & \multicolumn{3}{c}{$p=200$}\\
 \hline
 & $0\%$ & $50\%$ & $90\%$ & $0\%$ & $50\%$ & $90\%$  & $0\%$ & $50\%$ & $90\%$\\
   QC & 16.1(0.2) & 41.9(0.4) & 49.3(0.1) & 9.5(0.2) & 41.4(0.3) & 49.1(0.1) & 4.1(0.1) & 38(0.2) & 48.6(0.1) \\
  MC & 41.9(0.1) & 47(0.1) & 49.4(0) & 39(0.1) & 45.7(0.1) & 49.2(0.1) & 34.6(0.1) & 43.9(0.1) & 48.9(0.1) \\
  EMC & 39.9(0.2) & 44.3(0.2) & 48.7(0.1) & 37.3(0.2) & 43.1(0.1) & 48.4(0.1) & 33.2(0.1) & 40.5(0.1) & 47.8(0.1) \\
  EQC/LOGISTIC & 19.9(0.5) & 30.2(0.3) & 46(0.3) & 12.9(0.2) & 28.3(0.6) & 46.8(0.3) & 7.2(0.1) & 16.8(0.2) & 42.1(0.4) \\
  EQC/RIDGE & 14.8(0.1) & 25.7(0.2) & 45.5(0.3) & 7.8(0.1) & 17.6(0.2) & 42.1(0.3) & 3(0.1) & 10.1(0.1) & 37.5(0.3) \\
  EQC/LASSO & 18.7(0.2) & 28(0.3) & 45.5(0.3) & 14.4(0.2) & 21.6(0.2) & 43(0.4) & 12.8(0.2) & 16.8(0.2) & 38.9(0.5) \\
  EQC/LSVM & 19.4(0.2) & 30.6(0.3) & 45.8(0.3) & 12.5(0.3) & 25.5(0.3) & 45.1(0.3) & 4.5(0.1) & 14.7(0.2) & 41.7(0.4) \\
  NB & 49.2(0.1) & 49.3(0.1) & 49.7(0) & 49.1(0.1) & 49.3(0.1) & 49.7(0.1) & 49.1(0.1) & 49.1(0.1) & 49.5(0.1) \\
  LDA & 46(0.1) & 47.9(0.1) & 49.4(0.1) & 45.7(0.1) & 47.7(0.1) & 49.6(0.1) & 48.7(0.1) & 49.5(0.1) & 50(0.1) \\
  RIDGE & 46.2(0.1) & 47.9(0.1) & 49.6(0.1) & 45.1(0.1) & 47.2(0.1) & 49.4(0.1) & 43.2(0.1) & 46.2(0.1) & 49.1(0.1) \\
  LASSO & 49.8(0.1) & 49.8(0.1) & 50(0) & 49.7(0.1) & 49.8(0.1) & 50(0) & 49.6(0.1) & 49.7(0.1) & 49.9(0) \\
  LSVM & 46.1(0.1) & 47.8(0.1) & 49.4(0.1) & 45.9(0.1) & 47.8(0.1) & 49.6(0.1) & 44.6(0.1) & 47(0.1) & 49.3(0.1) \\
  RSVM & 42.8(0.2) & 46.8(0.1) & 49.4(0.1) & 43.3(0.1) & 46.8(0.1) & 49.5(0.1) & 41.5(0.1) & 46(0.1) & 49.1(0.1) \\
   \hline
\end{tabular}
\end{table}

\begin{table}[H]
\centering
\caption{ Simulation study:
the mean test classification error rates,
and their standard errors in parentheses,
of each method for the  \textbf{independent HETEROGENEOUS} scenario.
All numbers are in percentages and rounded to one digit.
The third line indicates the percentage of irrelevant variables within $p$ variables.}
\label{tab:simTabheterogeneous}
\scriptsize
\begin{tabular}{llllllllll}
\hline
\multicolumn{10}{c}{$n=100$} \\
 \hline & \multicolumn{3}{c}{$p=50$} & \multicolumn{3}{c}{$p=100$} & \multicolumn{3}{c}{$p=200$}\\
 \hline
 & $0\%$ & $50\%$ & $90\%$ & $0\%$ & $50\%$ & $90\%$  & $0\%$ & $50\%$ & $90\%$\\
QC & 23.7(0.4) & 38.8(0.4) & 48.2(0.1) & 18.3(0.3) & 34.7(0.3) & 47.8(0.1) & 10.3(0.2) & 29.1(0.3) & 46.8(0.1) \\
  MC & 37.2(0.2) & 43.4(0.2) & 48.7(0.1) & 31.7(0.2) & 40.8(0.2) & 48(0.1) & 25.4(0.2) & 37.1(0.1) & 47.4(0.1) \\
  EMC & 34.9(0.3) & 41.2(0.2) & 48.1(0.1) & 28.7(0.2) & 37.9(0.2) & 47.3(0.1) & 20.9(0.2) & 33(0.2) & 46(0.1) \\
  EQC/LOGISTIC & 9.3(0.3) & 22.4(0.7) & 43.4(0.4) & 6(0.2) & 13.5(0.3) & 38.7(0.4) & - & - & - \\
  EQC/RIDGE & 5(0.1) & 14(0.2) & 40.7(0.3) & 1.7(0) & 7.2(0.1) & 36.3(0.4) & 0.3(0) & 2.6(0.1) & 29.2(0.2) \\
  EQC/LASSO & 6.6(0.2) & 12.3(0.3) & 31.6(0.5) & 5.2(0.2) & 7.5(0.2) & 23.3(0.3) & 4.8(0.1) & 5.4(0.1) & 15.8(0.3) \\
  EQC/LSVM & 6(0.2) & 16.3(0.3) & 42.4(0.4) & 1.8(0.1) & 9.3(0.2) & 38.2(0.4) & 0.3(0) & 3.2(0.1) & 32(0.3) \\
  NB & 45.1(0.1) & 46.6(0.1) & 49(0.1) & 44.3(0.1) & 45.7(0.1) & 48.6(0.1) & 43.6(0.1) & 45(0.1) & 48.3(0.1) \\
  LDA & 41.4(0.2) & 45.1(0.2) & 48.9(0.1) & 46.7(0.3) & 48.1(0.2) & 49.6(0.1) & 38.1(0.2) & 43.5(0.2) & 48.4(0.1) \\
  RIDGE & 38.2(0.2) & 43.5(0.2) & 48.6(0.1) & 33.7(0.2) & 40.7(0.2) & 47.9(0.1) & 28.5(0.2) & 37.5(0.2) & 47.1(0.1) \\
  LASSO & 48.6(0.3) & 48.5(0.2) & 49.8(0.1) & 47.6(0.3) & 48.4(0.2) & 49.6(0.1) & 46.9(0.4) & 48.2(0.3) & 49.5(0.1) \\
  LSVM & 41.6(0.2) & 45.3(0.2) & 49(0.1) & 37.9(0.2) & 43.2(0.2) & 48.5(0.1) & 31.8(0.2) & 39.6(0.2) & 47.6(0.1) \\
  RSVM & 38.8(0.2) & 43.8(0.2) & 48.7(0.1) & 34.6(0.2) & 41.4(0.2) & 48(0.1) & 29.1(0.2) & 38.2(0.2) & 47.3(0.1) \\
  \hline
\multicolumn{10}{c}{$n=200$} \\
 \hline & \multicolumn{3}{c}{$p=50$} & \multicolumn{3}{c}{$p=100$} & \multicolumn{3}{c}{$p=200$}\\
 \hline
 & $0\%$ & $50\%$ & $90\%$ & $0\%$ & $50\%$ & $90\%$  & $0\%$ & $50\%$ & $90\%$\\
 QC & 17.6(0.3) & 33.1(0.4) & 47.8(0.1) & 11.8(0.2) & 29(0.2) & 47(0.2) & 4.9(0.1) & 21.9(0.1) & 45.1(0.2) \\
  MC & 34.5(0.1) & 41.4(0.2) & 48(0.1) & 28.5(0.1) & 38(0.1) & 47.5(0.1) & 21.2(0.1) & 33.1(0.1) & 46.4(0.1) \\
  EMC & 32.2(0.1) & 39.1(0.2) & 47.2(0.1) & 25.3(0.1) & 34.7(0.2) & 46.7(0.1) & 17.6(0.1) & 28.9(0.2) & 44.9(0.1) \\
  EQC/LOGISTIC & 7.1(0.8) & 19.6(1.1) & 36(0.2) & 2.9(0.1) & 9.1(0.2) & 34.8(0.4) & 2.3(0.1) & 5.6(0.1) & 27.8(0.2) \\
  EQC/RIDGE & 3.2(0.1) & 9.9(0.2) & 36.6(0.2) & 1.2(0) & 4.9(0.1) & 30.7(0.2) & 0.2(0) & 1.7(0) & 23.3(0.2) \\
  EQC/LASSO & 3.7(0.1) & 7.8(0.1) & 29.7(0.3) & 2.6(0.1) & 4.2(0.1) & 20.4(0.2) & 1.8(0) & 2.7(0.1) & 13.5(0.2) \\
  EQC/LSVM & 3.5(0.1) & 10.9(0.2) & 36.8(0.2) & 1.1(0) & 6.4(0.2) & 33.4(0.2) & 0.2(0) & 1.9(0) & 27.1(0.2) \\
  NB & 43.7(0.1) & 45.6(0.1) & 48.6(0.1) & 42.6(0.1) & 44.2(0.1) & 48.2(0.1) & 41.3(0.1) & 43.1(0.1) & 47.5(0.1) \\
  LDA & 37.5(0.2) & 42.7(0.2) & 48.1(0.1) & 35.8(0.2) & 41.5(0.2) & 48(0.1) & 46.7(0.2) & 47.9(0.2) & 49.4(0.1) \\
  RIDGE & 35.8(0.1) & 41.7(0.2) & 47.9(0.1) & 30.8(0.1) & 38.2(0.1) & 47.2(0.1) & 24.7(0.1) & 33.9(0.1) & 46.1(0.1) \\
  LASSO & 48.5(0.2) & 48.9(0.2) & 49.6(0.1) & 47.5(0.3) & 48.4(0.2) & 49.6(0.1) & 45.5(0.5) & 47.5(0.3) & 49.2(0.1) \\
  LSVM & 38(0.2) & 42.9(0.2) & 48.2(0.1) & 36.3(0.2) & 42.1(0.2) & 48.2(0.1) & 29.4(0.2) & 37.8(0.1) & 47.2(0.1) \\
  RSVM & 36(0.1) & 41.8(0.2) & 48(0.1) & 31.4(0.1) & 38.8(0.1) & 47.5(0.1) & 25.1(0.2) & 34.9(0.1) & 46.5(0.1) \\
   \hline
\end{tabular}
\end{table}

\begin{table}[H]
\centering
\caption{ Simulation study:
the mean test classification error rates,
and their standard errors in parentheses,
of each method for the  \textbf{dependent HETEROGENEOUS} scenario.
All numbers are in percentages and rounded to one digit.
The third line indicates the percentage of irrelevant variables within $p$ variables.}
\label{tab:simTabheterogeneousdep}
\scriptsize
\begin{tabular}{llllllllll}
\hline
\multicolumn{10}{c}{$n=100$} \\
 \hline & \multicolumn{3}{c}{$p=50$} & \multicolumn{3}{c}{$p=100$} & \multicolumn{3}{c}{$p=200$}\\
 \hline
 & $0\%$ & $50\%$ & $90\%$ & $0\%$ & $50\%$ & $90\%$  & $0\%$ & $50\%$ & $90\%$\\
QC & 23.9(0.4) & 38.5(0.4) & 48.2(0.1) & 18.2(0.2) & 34.5(0.2) & 47.9(0.1) & 10.5(0.2) & 28.8(0.2) & 47.1(0.2) \\
  MC & 37.7(0.2) & 43.4(0.2) & 48.6(0.1) & 32.7(0.2) & 41.1(0.2) & 48.2(0.1) & 26.3(0.2) & 37.3(0.2) & 47.3(0.1) \\
  EMC & 35.4(0.2) & 41.4(0.2) & 48(0.1) & 29.1(0.2) & 38.2(0.2) & 47.4(0.1) & 21.9(0.2) & 33.1(0.2) & 46.1(0.1) \\
  EQC/LOGISTIC & 8.5(0.3) & 21.9(0.8) & 44.3(0.5) & 6.2(0.2) & 12.6(0.3) & 39.6(0.4) & - & - & - \\
  EQC/RIDGE & 5.1(0.1) & 13.5(0.2) & 40.8(0.4) & 1.8(0) & 7(0.1) & 36.2(0.3) & 0.3(0) & 2.7(0.1) & 29.7(0.3) \\
  EQC/LASSO & 6.7(0.2) & 11.8(0.3) & 32.7(0.7) & 5.4(0.1) & 7.3(0.2) & 22.6(0.4) & 5(0.1) & 5.7(0.1) & 16(0.2) \\
  EQC/LSVM & 5.9(0.2) & 15.8(0.5) & 41.4(0.4) & 1.9(0.1) & 9.4(0.2) & 38.5(0.4) & 0.3(0) & 3.3(0.1) & 31.8(0.3) \\
  NB & 45.1(0.1) & 46.5(0.1) & 48.9(0.1) & 44.4(0.1) & 45.8(0.1) & 48.7(0.1) & 43.7(0.1) & 45.1(0.1) & 48.1(0.1) \\
  LDA & 42.4(0.3) & 45.3(0.2) & 49(0.1) & 47.1(0.3) & 48.4(0.2) & 49.6(0.1) & 38.7(0.3) & 43.5(0.2) & 48.4(0.1) \\
  RIDGE & 39.2(0.2) & 43.6(0.2) & 48.5(0.1) & 35(0.2) & 41.2(0.2) & 48(0.1) & 29.9(0.2) & 38.1(0.2) & 47.1(0.1) \\
  LASSO & 48.1(0.3) & 49.2(0.2) & 49.9(0.1) & 48.3(0.3) & 48.9(0.2) & 49.8(0.1) & 47.3(0.4) & 48.4(0.3) & 49.6(0.1) \\
  LSVM & 42.7(0.3) & 45.5(0.2) & 49.1(0.1) & 39.3(0.2) & 43.9(0.2) & 48.7(0.1) & 33.1(0.2) & 40.7(0.2) & 47.8(0.1) \\
  RSVM & 39.5(0.2) & 43.9(0.2) & 48.8(0.1) & 35.8(0.2) & 41.8(0.2) & 48.3(0.1) & 30.4(0.2) & 38.9(0.2) & 47.5(0.1) \\
  \hline
\multicolumn{10}{c}{$n=200$} \\
 \hline & \multicolumn{3}{c}{$p=50$} & \multicolumn{3}{c}{$p=100$} & \multicolumn{3}{c}{$p=200$}\\
 \hline
 & $0\%$ & $50\%$ & $90\%$ & $0\%$ & $50\%$ & $90\%$  & $0\%$ & $50\%$ & $90\%$\\
 QC & 17.8(0.3) & 32.6(0.4) & 47.4(0.2) & 12.1(0.2) & 28.8(0.2) & 46.7(0.2) & 5.3(0.1) & 22.3(0.2) & 45.1(0.2) \\
  MC & 34.9(0.2) & 41.9(0.1) & 48.2(0.1) & 29.4(0.2) & 38.4(0.1) & 47.3(0.1) & 22.2(0.1) & 34(0.1) & 46.5(0.1) \\
  EMC & 32.3(0.2) & 39.5(0.2) & 47.4(0.1) & 26(0.1) & 35.2(0.2) & 46.3(0.1) & 18.4(0.1) & 30.1(0.2) & 44.9(0.1) \\
  EQC/LOGISTIC & 6.8(0.6) & 20.6(1.2) & 35.5(0.2) & 3(0.1) & 9.3(0.5) & 34.2(0.4) & 2.4(0.1) & 5.3(0.1) & 27.9(0.4) \\
  EQC/RIDGE & 3.2(0.1) & 10.1(0.2) & 35.7(0.1) & 1.2(0) & 4.9(0.1) & 30.3(0.2) & 0.2(0) & 1.6(0) & 23.3(0.2) \\
  EQC/LASSO & 3.8(0.1) & 8(0.2) & 28.9(0.2) & 2.7(0.1) & 4.2(0.1) & 20(0.2) & 1.9(0.1) & 2.6(0.1) & 13.2(0.2) \\
  EQC/LSVM & 3.7(0.1) & 11(0.3) & 36.1(0.2) & 1.2(0.1) & 6.4(0.2) & 33.2(0.3) & 0.2(0) & 1.8(0) & 26.8(0.2) \\
  NB & 43.8(0.1) & 45.6(0.1) & 48.7(0.1) & 42.7(0.1) & 44.4(0.1) & 48.2(0.1) & 41.3(0.1) & 43.3(0.1) & 47.6(0.1) \\
  LDA & 38.2(0.2) & 43.5(0.1) & 48.4(0.1) & 36.9(0.2) & 42.2(0.2) & 48.1(0.1) & 46.7(0.2) & 47.7(0.2) & 49.4(0.1) \\
  RIDGE & 36.8(0.1) & 42.3(0.1) & 48.1(0.1) & 32.4(0.1) & 39.2(0.1) & 47.1(0.1) & 26.6(0.1) & 35.2(0.1) & 46.2(0.1) \\
  LASSO & 48.8(0.2) & 49(0.2) & 49.9(0.1) & 47.9(0.3) & 48.6(0.2) & 49.3(0.1) & 46.8(0.4) & 48.1(0.2) & 49.4(0.1) \\
  LSVM & 38.7(0.2) & 43.6(0.2) & 48.5(0.1) & 37.2(0.2) & 42.6(0.2) & 48.2(0.1) & 31.3(0.2) & 39.3(0.1) & 47.4(0.1) \\
  RSVM & 36.1(0.1) & 42.3(0.1) & 48.2(0.1) & 32.4(0.2) & 39.5(0.2) & 47.5(0.1) & 26.6(0.1) & 35.9(0.1) & 46.7(0.1) \\
   \hline
\end{tabular}
\end{table}

\newpage

\bibliographystyle{apa}
\bibliography{main}

\begin{thebibliography}{}

\bibitem[\protect\astroncite{Bickel and Levina}{2004}]{Bickel2004}
Bickel, P.~J. and Levina, E. (2004).
\newblock Some theory for {F}isher's linear discriminant function, `naive
  {B}ayes', and some alternatives when there are many more variables than
  observations.
\newblock {\em Bernoulli}, 10(6):989--1010.

\bibitem[\protect\astroncite{Breiman}{1996}]{breiman1996stacked}
Breiman, L. (1996).
\newblock Stacked regressions.
\newblock {\em Machine learning}, 24(1):49--64.

\bibitem[\protect\astroncite{Breiman}{2001}]{breiman2001random}
Breiman, L. (2001).
\newblock Random forests.
\newblock {\em Machine learning}, 45(1):5--32.

\bibitem[\protect\astroncite{Cardoso-Cachopo}{2007}]{phd-Ana-Cardoso-Cachopo}
Cardoso-Cachopo, A. (2007).
\newblock {Improving Methods for Single-label Text Categorization}.
\newblock PdD Thesis, Instituto Superior Tecnico, Universidade Tecnica de
  Lisboa.

\bibitem[\protect\astroncite{Cleveland}{1993}]{cleveland1993visualizing}
Cleveland, W.~S. (1993).
\newblock {\em Visualizing data}.
\newblock Hobart Press.

\bibitem[\protect\astroncite{Cortes and Vapnik}{1995}]{cortes1995support}
Cortes, C. and Vapnik, V. (1995).
\newblock Support-vector networks.
\newblock {\em Machine learning}, 20(3):273--297.

\bibitem[\protect\astroncite{Dietterich}{2000}]{dietterich2000ensemble}
Dietterich, T.~G. (2000).
\newblock Ensemble methods in machine learning.
\newblock In {\em International workshop on multiple classifier systems}, pages
  1--15. Springer.

\bibitem[\protect\astroncite{Dudoit et~al.}{2002}]{Dudoit2002}
Dudoit, S., Fridlyand, J., and Speed, T.~P. (2002).
\newblock Comparison of discrimination methods for the classification of tumors
  using gene expression data.
\newblock {\em Journal of the American Statistical Association},
  97(457):77--87.

\bibitem[\protect\astroncite{Fan and Fan}{2008}]{Fan2008}
Fan, J. and Fan, Y. (2008).
\newblock High dimensional classification using features annealed independence
  rules.
\newblock {\em Annals of statistics}, 36(6).

\bibitem[\protect\astroncite{Feinerer and Hornik}{2017}]{tm}
Feinerer, I. and Hornik, K. (2017).
\newblock tm: Text mining package.
\newblock \url{https://CRAN.R-project.org/package=tm}.
\newblock R package version 0.7-3.

\bibitem[\protect\astroncite{Freund and Schapire}{1997}]{freund1997decision}
Freund, Y. and Schapire, R.~E. (1997).
\newblock A decision-theoretic generalization of on-line learning and an
  application to boosting.
\newblock {\em Journal of computer and system sciences}, 55(1):119--139.

\bibitem[\protect\astroncite{Friedman et~al.}{2010}]{Friedman2010}
Friedman, J., Hastie, T., and Tibshirani, R. (2010).
\newblock Regularization paths for generalized linear models via coordinate
  descent.
\newblock {\em Journal of Statistical Software}, 33(1):1--22.

\bibitem[\protect\astroncite{Hall et~al.}{2009}]{HALL2009}
Hall, P., Titterington, D.~M., and Xue, J.-H. (2009).
\newblock Median-based classifiers for high-dimensional data.
\newblock {\em Journal of the American Statistical Association},
  104(488):1597--1608.

\bibitem[\protect\astroncite{Hastie et~al.}{2009}]{tibs2009}
Hastie, T., Tibshirani, R., and Friedman, J. (2009).
\newblock {\em The Elements of Statistical Learning}.
\newblock Springer Series in Statistics. Springer-Verlag, New York., 2 edition.

\bibitem[\protect\astroncite{Hennig and Viroli}{2016a}]{Hennig2016}
Hennig, C. and Viroli, C. (2016a).
\newblock Quantile-based classifiers.
\newblock {\em Biometrika}, 103(2):435--446.

\bibitem[\protect\astroncite{Hennig and Viroli}{2016b}]{quantileDA}
Hennig, C. and Viroli, C. (2016b).
\newblock quantileda: Quantile classifier.
\newblock \url{https://CRAN.R-project.org/package=quantileDA}.
\newblock R package version 1.1.

\bibitem[\protect\astroncite{James et~al.}{2013}]{James2013}
James, G., Witten, D., Hastie, T., and Tibshirani, R. (2013).
\newblock {\em An Introduction to Statistical Learning}.
\newblock Springer Series in Statistics. Springer-Verlag, New York.

\bibitem[\protect\astroncite{Joe}{2006}]{JOE2006}
Joe, H. (2006).
\newblock Generating random correlation matrices based on partial correlations.
\newblock {\em Journal of Multivariate Analysis}, 97(10):2177 -- 2189.

\bibitem[\protect\astroncite{Koenker}{2005}]{Koenker2005}
Koenker, R. (2005).
\newblock {\em Quantile Regression}.
\newblock Econometric Society Monographs. Cambridge University Press.

\bibitem[\protect\astroncite{Koenker and Bassett}{1978}]{Koenker1978}
Koenker, R. and Bassett, G. (1978).
\newblock Regression quantiles.
\newblock {\em Econometrica}, 46(1):33--50.

\bibitem[\protect\astroncite{Kuhn and Johnson}{2013}]{kuhn2013}
Kuhn, M. and Johnson, K. (2013).
\newblock {\em Applied Predictive Modeling}.
\newblock Springer.

\bibitem[\protect\astroncite{Lai and McLeod}{2018}]{eqcGithub}
Lai, Y. and McLeod, A.~I. (2018).
\newblock eqc: Ensemble quantile classifier.
\newblock \url{https://github.com/CliffordLai/eqc}.
\newblock R package version 1.0-5.

\bibitem[\protect\astroncite{Lewis}{1997}]{Lewis1997}
Lewis, D. (1997).
\newblock Reuters-21578 text categorization collection distribution 1.0.

\bibitem[\protect\astroncite{Lior}{2019}]{lior}
Lior, R. (2019).
\newblock {\em Ensemble Learning: Pattern Classification Using Ensemble
  Methods}.
\newblock World Scientific Publishing Company, 2 edition.

\bibitem[\protect\astroncite{Mason}{1982}]{Mason1982}
Mason, D.~M. (1982).
\newblock Some characterizations of almost sure bounds for weighted
  multidimensional empirical distributions and a {G}livenko-{C}antelli theorem
  for sample quantiles.
\newblock {\em Zeitschrift f{\"u}r Wahrscheinlichkeitstheorie und Verwandte
  Gebiete}, 59(4):505--513.

\bibitem[\protect\astroncite{Meyer et~al.}{2018}]{e1071}
Meyer, D., Dimitriadou, E., Hornik, K., Weingessel, A., and Leisch, F. (2018).
\newblock e1071: Misc functions of the department of statistics, probability
  theory group (formerly: E1071), tu wien.
\newblock \url{https://CRAN.R-project.org/package=e1071}.
\newblock R package version 1.7-0.

\bibitem[\protect\astroncite{Newbold and Granger}{1974}]{Newbold}
Newbold, P. and Granger, C. W.~T. (1974).
\newblock Experience with forecasting univariate time series and the
  combination of forecasts.
\newblock {\em Journal of the Royal Statistical Society A}, 137(2):131--165.

\bibitem[\protect\astroncite{Park and Hastie}{2007}]{park2007penalized}
Park, M.~Y. and Hastie, T. (2007).
\newblock Penalized logistic regression for detecting gene interactions.
\newblock {\em Biostatistics}, 9(1):30--50.

\bibitem[\protect\astroncite{Qiu and Joe.}{2015}]{clusterGeneration}
Qiu, W. and Joe., H. (2015).
\newblock clustergeneration: Random cluster generation (with specified degree
  of separation).
\newblock R package version 1.3.4.

\bibitem[\protect\astroncite{Schapire and Freund}{2012}]{schapire2012}
Schapire, R. and Freund, Y. (2012).
\newblock {\em Boosting: Foundations and Algorithms}.
\newblock MIT Press.

\bibitem[\protect\astroncite{Sebastiani}{2002}]{Sebastiani2002}
Sebastiani, F. (2002).
\newblock Machine learning in automated text categorization.
\newblock {\em ACM Comput. Surv.}, 34(1):1--47.

\bibitem[\protect\astroncite{Silver}{2012}]{silver2012}
Silver, N. (2012).
\newblock {\em The Signal and the Noise}.
\newblock Penguin Publishing Group.

\bibitem[\protect\astroncite{Tibshirani et~al.}{2003}]{TIBSHIRANI2003}
Tibshirani, R., Hastie, T., Narasimhan, B., and Chu, G. (2003).
\newblock Class prediction by nearest shrunken centroids, with applications to
  {DNA} microarrays.
\newblock {\em Statistical Science}, 18(1):104--117.

\bibitem[\protect\astroncite{Ting and Witten}{1999}]{ting1999issues}
Ting, K.~M. and Witten, I.~H. (1999).
\newblock Issues in stacked generalization.
\newblock {\em Journal of artificial intelligence research}, 10:271--289.

\bibitem[\protect\astroncite{Venables and Ripley}{2002}]{MASS}
Venables, W.~N. and Ripley, B.~D. (2002).
\newblock {\em Modern Applied Statistics with S}.
\newblock Springer, New York, fourth edition.
\newblock ISBN 0-387-95457-0.

\bibitem[\protect\astroncite{Wolpert}{1992}]{wolpert1992stacked}
Wolpert, D.~H. (1992).
\newblock Stacked generalization.
\newblock {\em Neural networks}, 5(2):241--259.

\bibitem[\protect\astroncite{Zhou}{2012}]{zhou2012ensemble}
Zhou, Z.-H. (2012).
\newblock {\em Ensemble methods: foundations and algorithms}.
\newblock Chapman and Hall/CRC.

\end{thebibliography}
\end{document}